\documentclass[runningheads]{llncs}

\usepackage[T1]{fontenc}

\usepackage{xcolor}

\usepackage{graphicx}

\usepackage{booktabs}

\usepackage[misc]{ifsym}

\usepackage{amsmath}
\usepackage{amsfonts}
\usepackage{amssymb}
\usepackage{algorithm}
\usepackage{algorithmic}

\usepackage[switch]{lineno}

\usepackage{tabularx}
   
\usepackage{subcaption}

\newtheorem{assumption}{Assumption}

\usepackage{mwe}

\begin{document}

\title{Differentially Private Sparse Linear Regression with Heavy-tailed Responses}

\author{
Xizhi Tian\inst{1,2,3},
Meng Ding \inst{4},
Touming Tao \inst{5}, \\
Zihang Xiang\inst{1,2},
Di Wang\inst{\dagger,1,2}
}

\authorrunning{X.Tian et al.}

\institute{Provable Responsible AI and Data Analytics (PRADA) Lab 
\and
King Abdullah University of Science and Technology
\and
Utrecht University
\and
University at Buffalo
\and 
Technical University Berlin}

\maketitle              
\def\thefootnote{$\dagger$}\footnotetext{Corresponding Author.} 
\begin{abstract}
As a fundamental problem in machine learning and differential privacy (DP), DP linear regression has been extensively studied. However, most existing methods focus primarily on either regular data distributions or low-dimensional cases with irregular data. To address these limitations, this paper provides a comprehensive study of DP sparse linear regression with heavy-tailed responses in high-dimensional settings.
In the first part, we introduce the DP-IHT-H method, which leverages the Huber loss and private iterative hard thresholding to achieve an estimation error bound of 
\(
    \tilde{O}\biggl( 
        s^{* \frac{1 }{2}} 
        \cdot \biggl(\frac{\log d}{n}\biggr)^{\frac{\zeta}{1 + \zeta}} 
        + 
        s^{* \frac{1 + 2\zeta}{2 + 2\zeta}} 
        \cdot \biggl(\frac{\log^2 d}{n \varepsilon}\biggr)^{\frac{\zeta}{1 + \zeta}} 
    \biggr)
\)
under the $(\varepsilon, \delta)$-DP model, where $n$ is the sample size, $d$ is the dimensionality, $s^*$ is the sparsity of the parameter, and $\zeta \in (0, 1]$ characterizes the tail heaviness of the data.
In the second part, we propose DP-IHT-L, which further improves the error bound under additional assumptions on the response and achieves 
\(
    \tilde{O}\Bigl(\frac{(s^*)^{3/2} \log d}{n \varepsilon}\Bigr).
\)
Compared to the first result, this bound is independent of the tail parameter $\zeta$. 
Finally, through experiments on synthetic and real-world datasets, we demonstrate that our methods outperform standard DP algorithms designed for ``regular'' data.

\end{abstract}

\section{Introduction}

Differential Privacy~(DP) \cite{dwork2006calibrating} has received significant attention and is now widely considered the \emph{de facto} standard to protect privacy in data analysis. DP provides a rigorous mathematical framework to ensure that the inclusion or exclusion of any single individual's data in a dataset does not significantly affect the output of an analysis, thereby preserving privacy. A large body of research has explored various DP guarantees, and these concepts have been successfully adopted in industry~\cite{ding2017telemetry,tang2017privacy}.

Linear regression in the DP model has been extensively studied for many years, becoming one of the most thoroughly explored topics in machine learning and DP communities. A substantial body of research has addressed the problem from various perspectives. Early investigations of DP linear regression were closely tied to more general frameworks such as DP Stochastic Convex Optimization~(DP-SCO) and Empirical Risk Minimization~(DP-ERM), as explored in the seminal works of~\cite{chaudhuri2009privacy,chaudhuri2011differentially,wang2017differentially,wang2019differentially,wang2019differentially1,wang2020escaping,huai2020pairwise,xue2021differentially,su2022faster,su2024faster,ding2024revisiting,zhang2025improved,shen2023differentially}. 
Subsequently, numerous methods have been proposed to address DP linear regression under different settings. For instance, several studies~\cite{bassily2020stability,feldman2020private,iyengar2019towards,wang2022differentially,ding2025nearly} focus on low-dimensional scenarios in the central DP model, while others~\cite{cai2020cost,kasiviswanathan2016efficient,kifer2012private,talwar2015nearly,wang2020knowledge} extend these ideas to high-dimensional sparse linear regression—an increasingly relevant setting for modern, structured data. The local DP model has also attracted attention~\cite{duchi2018minimax,wang2020empirical,wang2019sparse,wang2020sparse,zhu2023improved,zhu2024truthful}, imposing stricter privacy requirements but potentially offering stronger protection.

Despite these advances, most prior work assumes \emph{regular} data distributions, typically requiring that features and responses are bounded or sub-Gaussian. Classical approaches such as output perturbation~\cite{chaudhuri2011differentially} or objective/gradient perturbation~\cite{bassily2019private} rely on these assumptions to guarantee an $O(1)$-Lipschitz loss for all data points. In practice, however, these conditions often fail—particularly in domains like biomedicine and finance, where heavy-tailed distributions commonly arise~\cite{biswas2007statistical,ibragimov2015heavy,woolson2011statistical}. In such cases, Lipschitz conditions may be violated; for example, in linear regression with squared loss $\ell(w,(x,y)) = (w^\top x - y)^2$, heavy-tailed features can lead to unbounded gradients, invalidating assumptions used in many DP proofs. Although gradient truncation or trimming~\cite{abadi2016deep} has been suggested as a remedy, a thorough analysis of its convergence properties under DP constraints has been lacking. This gap underscores the need for private and robust methods specifically tailored to heavy-tailed data.

Recent works have begun addressing DP linear regression in the presence of heavy-tailed data~\cite{barber2014privacy,kamath2020private,liu2021robust,wang2020differentially,wang2025private}. However, some of these approaches suffer from a polynomial dependence on the data dimension \(d\), making them unsuitable for high-dimensional settings where \(d \gg n\). Thus, a natural question is: What are the theoretical behaviors of DP linear regression in high-dimensional sparse cases with heavy-tailed data?


\noindent {\bf Our Contributions.} In this paper, we study the setting of high-dimensional sparse linear regression with heavy-tailed response under DP constraints. Specifically, we consider the scenario where the feature vector is sub-Gaussian and the response only has a finite $(1+\zeta)$-th moment with $\zeta \in (0, 1]$. We propose novel DP linear regression algorithms that are robust to heavy-tailed data and capable of handling high-dimensional sparse problems effectively. Specifically:

\begin{enumerate}
  \item For general heavy-tailed responses, we propose the DP-IHT-H algorithm, addressing the unexplored challenge of adapting Huber loss-based linear regression to differential privacy. Specifically, it achieves an error bound under $(\epsilon, \delta)$-DP given by
    \[
       \tilde{O}\biggl( 
        s^{* \frac{1}{2} }
        \cdot \biggl(\frac{\log d}{n}\biggr)^{\frac{\zeta}{1 + \zeta}} 
        + 
        s^{* \frac{1 + 2\zeta}{2 + 2\zeta}} 
        \cdot \biggl(\frac{\log^2 d}{n \varepsilon}\biggr)^{\frac{\zeta}{1 + \zeta}} 
    \biggr),
    \]
    where $n$ is the sample size, $d$ is the data dimensionality, and $s^*$ represents the sparsity of the parameter.

    \item To further improve the error bound,  we propose the DP-IHT-L algorithm, which leverages the \(\ell_1\) loss function instead. By adding some mild assumptions on the response noise, the DP-IHT-L algorithm achieves a lower gradient bound and reduces the magnitude of the added noise. This enhancement leads to improved error bounds and greater stability, regardless of the value of \(\zeta\). Specifically, under \((\epsilon, \delta)\)-DP, the error is bounded by  
\(
\tilde{O}\Bigl(\frac{(s^*)^{3/2} \log d}{n \epsilon}\Bigr), 
\)  which is independent of the moment and matches best-known results of the sub-Gaussian case \cite{hu2022high,cai2021cost}. 
    \item Through extensive experiments on both synthetic and real-world datasets, we demonstrate that our methods outperform DP algorithms designed for regular data. Moreover, in some cases, DP-IHT-L further improves upon DP-IHT-H.
\end{enumerate}

\section{Related Work}\label{sec2}
As we mentioned, DP linear regression has been extensively studied. Still, most existing methods assume that the underlying data distribution is sub-Gaussian or bounded, rendering them unsuitable for heavy-tailed data. In contrast, in the non-private setting, recent advances have addressed Stochastic Convex Optimization~(SCO) and Empirical Risk Minimization~(ERM) under heavy-tailed data~\cite{brownlees2015empirical,holland2019better,lecue2018robust,lugosi2019risk,sun2020adaptive,shen2023computationally}. However, these non-private methods are not directly adaptable to private settings, particularly in our high dimensional sparse scenarios.

The first study addressing DP-SCO with heavy-tailed data was proposed by~\cite{wang2020differentially}, which introduced three methods based on distinct assumptions. The first method utilizes the Sample-and-Aggregate framework~\cite{nissim2007smooth}, but its stringent assumptions lead to a relatively large error bound. The second method leverages smooth sensitivity~\cite{bun2019average} but requires the data distribution to be sub-exponential. Additionally, \cite{wang2020differentially} proposed a private estimator inspired by robust statistics, which shares similarities with our approach.  Building on the mean estimator proposed in~\cite{kamath2020private}, \cite{kamath2021improved,tao2022private} recently investigated DP-SCO and achieved improved (expected) excess population risks of \(\tilde{O}\left(\left(\frac{d}{\epsilon n}\right)^{\frac{1}{2}}\right)\) and \(\tilde{O}\left(\frac{d}{\epsilon n}\right)\) for convex and strongly convex loss functions, respectively. These results rely on the assumption that the gradient of the loss function has a bounded second-order moment, aligning with the best-known outcomes for heavy-tailed mean estimation. However, these approaches only consider low dimensional case and  cannot address high-dimensional or sparse learning problems. Furthermore, their methods are not directly extendable to our models, where the assumption of bounded second-order moments of the loss function gradient is overly restrictive than our assumption on the finite $(1+\zeta)$-moment for the response. 

Recently, \cite{hu2022high} proposed a method that requires only that the distributions of \(x\) sub-Gaussian and \(y\) at least have bounded second-order moments, achieving an error bound of \(\tilde{O}\left(\frac{(s^*)^{3} \log^{2} d}{n \epsilon}\right)\). In our paper, we consider the weaker assumption that \(y\) only has the bounded $(1+\zeta)$-th moment. We show that, under some additional assumptions, it is possible to achieve almost the same error as in \cite{hu2022high}.

\section{Preliminaries}

\begin{definition}[Differential Privacy \cite{dwork2006calibrating}]
    A randomized mechanism \( M \) for the data universe \( \mathcal{D} \) satisfies \( (\varepsilon, \delta) \)-differential privacy if, for all measurable subsets \( S \subseteq \operatorname{Range}(M) \) and for all pairs of adjacent datasets \( D, D' \in \mathcal{D} \) (differing by at most one data point),
    \(
    \Pr\bigl[M(D) \in S\bigr] \;\le\; e^\varepsilon\,\Pr\bigl[M(D') \in S\bigr] \;+\;\delta,
    \)
    where the probability space is over the randomness of the mechanism \( M \).
\end{definition}

\begin{definition}[Laplacian Mechanism]
 Given a function $q: X^n \rightarrow \mathbb{R}^d$, the Laplacian Mechanism is defined as:
 \(
 \mathcal{M}_L(D, q, \epsilon)=q(D)+\left(Y_1, Y_2, \cdots, Y_d\right),
 \)
 where $Y_i$ is i.i.d. drawn from a Laplacian Distribution $\operatorname{Lap}\left(\frac{\Delta_1(q)}{\epsilon}\right)$, and $\Delta_1(q)$ is the $\ell_1$-sensitivity of the function $q$, i.e.,
 \(
 \Delta_1(q)=\sup _{D \sim D^{\prime}}\left\|q(D)-q\left(D^{\prime}\right)\right\|_1.
 \)
 For a parameter $\lambda$, the Laplacian distribution has the density function $\operatorname{Lap}(\lambda)(x)=\frac{1}{2 \lambda} \exp \left(-\frac{|x|}{\lambda}\right)$. The Laplacian Mechanism preserves $\epsilon$-DP.
\end{definition}

\begin{definition}[Sub-Gaussian Vector]
    A random vector \(\mathbf{X} \in \mathbb{R}^d\) is sub-Gaussian if every one-dimensional projection \(\langle \mathbf{X}, \mathbf{v} \rangle\), for any unit vector \(\mathbf{v} \in \mathbb{R}^d\), is a sub-Gaussian random variable. That is, there exists \(\sigma^2 > 0\) such that for all \(\lambda \in \mathbb{R}\),
    \(
    \mathbb{E}\left[e^{\lambda (\langle \mathbf{X}, \mathbf{v} \rangle - \mathbb{E}[\langle \mathbf{X}, \mathbf{v} \rangle])}\right] \leq e^{\frac{\lambda^2 \sigma^2}{2}}.
    \)
\end{definition}

We consider a sparse linear model \( y_i = x_i^\top \boldsymbol{\beta}^* + \varepsilon_i \), where \( \boldsymbol{\beta}^* \in \mathbb{R}^d \) is the underlying unknown parameter vector, and \( \varepsilon_i \) represents the noise term. In the high dimensional setting, we assume \( \boldsymbol{\beta}^* \)  is $s^* $-sparse, {\em i.e., }  \( s^* = |\operatorname{supp}(\boldsymbol{\beta}^*)| \) with \( s^* \ll d \). In DP linear regression, given a dataset where each sample is i.i.d. sampled from the linear model, the goal is to develop some $(\epsilon, \delta)$-DP estimator $\boldsymbol{\beta}^{priv}$ to make the estimation error $\|\boldsymbol{\beta}^{priv}-\boldsymbol{\beta}^*\|_2$ as small as possible.  We adopt the following general assumptions, requiring bounded eigenvalues and a bounded \( \boldsymbol{\beta}^* \), which are common in high dimensional setting \cite{shen2023computationally,sun2020adaptive,cai2020cost} :

\begin{assumption}\label{assumption1}
We assume the following:
\begin{enumerate}
    \item The covariates \(\{x_i\}\) are zero-mean $O(1)$-sub-Gaussian with covariance matrix \(\boldsymbol{\Sigma}=\mathbb{E}[xx^\top]\). The eigenvalues of \(\boldsymbol{\Sigma}\) are bounded as follows:
    \(
    c_l \;\le\; \lambda_{\min}(\boldsymbol{\Sigma}) \;\le\; \lambda_{\max}(\boldsymbol{\Sigma}) \;\le\; c_u.
    \)
    \item \(\|\boldsymbol{\beta}^*\|_2 \,\leq\, c_u^{1/2} r \;=\; L\), where $r$ is a constant.
\end{enumerate}
\end{assumption}
 In this paper, we focus on linear models with heavy-tailed responses. Unlike previous work on  the DP linear model, here we only assume the response (or the noise) has only bounded $1+\zeta$-th moment: 

\begin{assumption}\label{assumption2}
We assume the  noise \(\varepsilon_i\) has zero mean, \(\mathbb{E}[\varepsilon_i] = 0\), and a finite \((1+\zeta)\)-th moment with some $\zeta \in (0,1] $: 
    \begin{equation*}
    v_\zeta \;=\; \dfrac{1}{n} \sum_{i=1}^n \mathbb{E}\bigl(\lvert\varepsilon_i\rvert^{1+\zeta}\bigr) \;<\; \infty.
    \end{equation*}
\end{assumption}

\section{DP Iterative Hard Thresholding with Huber Loss}
\label{sec:dp-iht-h}

In the non-private case, recently \cite{tong2023functional,sun2020adaptive} show that an estimator based on the Huber loss \cite{huber1964robust} can achieve the optional estimation rate:
\begin{equation}\label{eq:huber}
\mathcal{L}_\tau(\boldsymbol{\beta}) 
\;=\; 
\dfrac{1}{n} 
\sum_{i=1}^n 
\ell_\tau\!\bigl(y_i - x_i^\top \boldsymbol{\beta}\bigr), {\text s.t. } \|\boldsymbol{\beta}\|_0\leq s, \|\boldsymbol{\beta}\|_2\leq L, 
\end{equation}
where $s$ is a parameter and  the Huber loss with parameter $\tau$ is defined as 
\begin{equation*}
\ell_\tau(x) \;:=\;
\begin{cases}
\dfrac{x^2}{2}, & \text{if } |x| < \tau, \\[6pt]
\tau\,|x|\; - \;\dfrac{\tau^2}{2}, & \text{otherwise}.
\end{cases}
\end{equation*}
Compared to the squared loss in the classical linear model, 
Huber loss is more robust to outliers and heavy-tailed noise, making it highly effective in non-DP scenarios. However, with the DP constraint, it is difficult to privatize such an estimator due to the following key challenges:
1). The optimization problem (non-convex and non-smooth) associated with Huber loss (\ref{eq:huber}) lacks an efficient algorithm for the solution. For instance, \cite{sun2020adaptive} proposes an adaptive Huber loss estimator but does not provide an efficient algorithm to solve it.
2). If we directly use the previous DP methods to the Huber loss such as objective perturbation~\cite{chaudhuri2009privacy} or gradient perturbation methods \cite{bassily2019private}, then we need to introduce a noise with scale $\Omega(d)$, which is extremely large as we consider the high dimensional setting where $d\gg n$. 

To address these challenges, we propose the DP-IHT-H algorithm. This algorithm:
1). Efficiently leverages the Huber loss to perform a linear regression.
2). Achieves $(\epsilon, \delta)$-DP guarantees with an error bound that only logarithmically depends on the dimension $d$.
In our DP-IHT-H algorithm, we first shrink the original feature vector \(x\) to make it have bounded $\ell_\infty$-norm. Next, we calculate the gradient on the shrunken data using the Huber loss and update our vector via gradient descent. Due to the bounded gradient of Huber loss, the error bound can be controlled with high probability, even for heavy-tailed data. Finally, we perform a "Peeling" step \cite{cai2021cost} to select the top \(s\) indices with the largest magnitudes in the vector while preserving DP.

The Peeling algorithm achieves privacy protection by adding Laplace noise twice. First, noise is added to each entry of the vector to perturb the original data, ensuring that no single component can be directly exposed during the selection process. Based on the noisy values, the algorithm iteratively selects the indices corresponding to the \(s\) largest magnitudes. After selecting the \(s\) indices, additional Laplace noise is added to further obscure the true values of the selected components. The final output is a sparse vector where only the selected components retain their perturbed values, while all other components are set to zero. Compared to using Gaussian noise, this approach results in a smaller noise scale, effectively protecting privacy while maintaining higher accuracy and output quality. In the following, we will provide the privacy and utility guarantees of DP-IHT-H. 

\begin{algorithm}[!tb]
\caption{DP IHT with Huber Loss (DP-IHT-H)}
\label{alg:dp-iht-h}
\textbf{Input}:  \(n\)-size dataset \(\mathbf{D}=\{(y_i, x_i)\}_{i \in [n]}\), Step size \(\eta\), Sparsity level \(s\), Privacy parameters \(\varepsilon, \delta\), Huber loss parameter \(\tau\), Truncation parameter \(K\), Total steps \(T\).

\textbf{Output}: \(\boldsymbol{\beta}^T\)

\begin{algorithmic}[1]
    \STATE Initialize \(\boldsymbol{\beta}^0 = 0\).
    \STATE \textbf{Clipping:} For each \(i \in [n]\), define the truncated sample \(\tilde{x}_i \in \mathbb{R}^d\) by 
    \[
    \tilde{x}_{i, j} = \operatorname{sign}(x_{i, j}) \min \!\Bigl\{\lvert x_{i, j} \rvert, K \Bigr\}, 
    \quad \forall j \in [d].
    \]
    \STATE Denote the truncated dataset as \(\tilde{D}=\{(\tilde{x}_i, \tilde{y}_i)\}_{i=1}^n\).
    \STATE Split the data \(\tilde{D}\) into \(T\) parts \(\{\tilde{D}_t\}_{t=1}^T\), each with \(m=\frac{n}{T}\) samples.
    \FOR{\(t = 0\) to \(T-1\)}
        \STATE \(\displaystyle \boldsymbol{\beta}^{t+0.5} 
        = \boldsymbol{\beta}^t \;-\; \frac{\eta}{m} \,\sum_{i\in \tilde{D}_t} \ell_\tau^{\prime}\left(y_i-\tilde{x}_i^{\top} \boldsymbol{\beta}^t\right) \tilde{x}_i\) 
        \STATE \(\displaystyle \boldsymbol{\beta}^{t+1} 
            = 
            \Pi_L\!\Bigl(\text{Peeling}\!\bigl(\boldsymbol{\beta}^{t+0.5}, \tilde{D}_t, s, \varepsilon, \delta, \tfrac{\eta \tau K}{m}\bigr)\Bigr)\), where $\Pi_L$ is the projection onto the $L$-radius ball. 
    \ENDFOR
    \STATE \textbf{return} \(\boldsymbol{\beta}^T\)
\end{algorithmic}
\end{algorithm}

\begin{algorithm}[!tb]
\caption{Peeling Procedure}
\label{alg:peeling}
\textbf{Input}: 
Vector \(\boldsymbol{x} \in \mathbb{R}^d\) (based on \(\mathbf{D}\)), Sparsity \(s\), Privacy parameters \(\varepsilon, \delta\), Noise scale \(\lambda\).

\textbf{Output}: \(\boldsymbol{x}\big|_S + \tilde{\boldsymbol{w}}_S\)

\begin{algorithmic}[1]
    \STATE Initialize \(S = \emptyset\).
    \FOR{\(i = 1\) to \(s\)}
        \STATE Generate noise \(\boldsymbol{w}_i \in \mathbb{R}^d\) where each component is drawn from:
        \(
        w_{i,j} \sim 
        \operatorname{Lap}\Bigl(\frac{2\lambda\sqrt{3s\log(1/\delta)}}{\varepsilon}\Bigr).
        \)
        \STATE Append 
        \(j^* = \arg\max_{j \in [d] \setminus S}\!\Bigl(|\boldsymbol{x}_j| + w_{i,j}\Bigr)\) to \(S\).
    \ENDFOR
    \STATE Generate \(\tilde{\boldsymbol{w}} \in \mathbb{R}^d\) where each component is drawn from:
    \(
    \tilde{w}_j \sim \operatorname{Lap}\Bigl(\frac{2\lambda\sqrt{3s\log(1/\delta)}}{\varepsilon}\Bigr).
    \)
    \STATE \textbf{return} \(\boldsymbol{x}\big|_S + \tilde{\boldsymbol{w}}_S\).
\end{algorithmic}
\end{algorithm}

\begin{theorem}
\label{thm:dp}
For any $0<\epsilon, 0<\delta<1$, DP-IHT-H is \((\varepsilon,\delta)\)-DP.
\end{theorem}

\begin{theorem}\label{theorem2}
If Assumption~\ref{assumption1} and ~\ref{assumption2} hold and assume $n$ is sufficiently large such that 
\(n \geq  \tilde{\Omega}\bigl((s^*)^{\tfrac{2}{3}}\log d\bigr)\). 
Setting 
\(s = O\bigl(s^*\bigr),\)
stepsize \(\eta = O(1)\), $K = O\bigl(\log d\bigr)$, \(T = O\bigl(\log n\bigr)\), and
\( \tau = O\Bigl(\Bigl(\frac{n}{T\,{(s^*)}^{\tfrac{3}{2}}\log^{\tfrac12}\!\bigl(\tfrac{1}{\sigma}\bigr)\,\log d}\Bigr)^{\tfrac{\zeta}{1+\zeta}}\Bigr)\) in Algorithm \ref{alg:dp-iht-h}, then with probability at least 
\(1 - O\Bigl(d^{-1} + T e^{-c_1 n}\Bigr)\) for a constant $c_1$,
we obtain the following bound on the final estimate:
\[
\|\boldsymbol{\beta}^T-\boldsymbol{\beta}^*\|_2 = O\Biggl( (s^*)^{\frac{1}{2}} \left(\frac{\log d}{n}\right)^{\frac{\zeta}{1+\zeta}} 
+ (s^*)^{\frac{1}{2}} \left(\frac{(s^*)^{\frac{1}{2}}\,\log^2 d\,\log(1/\delta)}{n\,\varepsilon}\right)^{\frac{\zeta}{1+\zeta}} \Biggr).
\]
\end{theorem}

\noindent
\textbf{Remark.} 
There are two terms in the above error bound. 
The first one corresponds to the optimal error bound in the non-private case~\cite{sun2020adaptive}. The second term corresponds to the error due to the noises added to ensure DP. 
    In the case $\epsilon=O(1)$ the overall error bound is dominated by the second term.  Moreover,  when \(\zeta = 1\), i.e., the noise has bounded variance, the error rate becomes to 
\(
\tilde{O}\!\bigl(\frac{(s^*)^{3/4}\,\log d}{\sqrt{n\,\varepsilon}}\bigr).
\)

Under similar conditions, that is, assume that \( x \) is sub-Gaussian and \( y \) has a finite \( 2\zeta \) -th moment with $\zeta\geq 1$ -\cite{hu2022high} establishes an upper bound for the privacy component, given by
\(
\tilde{O}\!\left((s^*)^{\frac{1+2\zeta}{1+\zeta}} \cdot \left(\frac{\log^3 n\,\log^2 d}{n^2\epsilon^2}\right)^{\frac{\zeta}{1+\zeta}}\right).
\) Thus, our results can be considered as an extension to the $(1+\zeta)$-th moment case.  For sub-Gaussian \( x \) and \( y \) a related bound is also provided in \cite{cai2021cost}:
\(
\tilde{O}\left(\sqrt{\frac{s^* \log d}{n}}+\frac{s^* \log d \sqrt{\log ^3 n}}{n \varepsilon}\right). 
\)
However, this result is not directly comparable to ours due to the strong assumptions imposed on the covariance matrix.

\section{DP Iterative Hard Thresholding with $\ell_1$ Loss}
\label{sec:dp_iht_l}

In DP-IHT-H, the algorithm achieves its lowest error bound when \(\zeta = 1\). This naturally raises the question: Can we refine the algorithm so that its performance becomes independent of \(\zeta\) even when $\zeta<1$?  In this section, we propose a new method, DP-IHT-L, to address the shortcomings of DP-IHT-H. The primary issue with DP-IHT-H lies in the dependence of the bounded gradient of the Huber loss on \(\zeta\), which results in larger noise being introduced during the ``Peeling'' step as the best value depends on $n$ and $\zeta$ (see Theorem \ref{theorem2}). Thus, it is necessary to use other loss functions that do not have such a parameter $\tau$ and are with a constant gradient bound. Motivated by \cite{shen2023computationally}, in the following we will show the  \(\ell_1\) loss (absolute loss function) satisfies this requirement. See Algorithm \ref{alg:dp-iht-l} for details. 

The basic idea of DP-IHT-L is similar to that of DP-IHT-H. First, we clip \(x\) to ensure it has an $\tilde{O}(1)$ bounded $\ell_\infty$-norm. Then we update \(\boldsymbol{\beta}^t\), but instead of using the gradient of the Huber loss, we replace it with the gradient of the \(\ell_1\) loss, thereby avoiding the dependence on the bound related to \(\zeta\). Finally, we apply the ``Peeling'' algorithm to introduce noise and preserve differential privacy.

\begin{theorem}
\label{thm:dp-l}
For $0<\epsilon$ and $0<\delta<1$, the DP-IHT-L algorithm is \((\varepsilon,\delta)\)-DP.
\end{theorem}
To get our bound, we pose additional assumptions on the noise. 
\begin{assumption}\label{assumption3}
We assume the following:
Let the noise terms \(\{\varepsilon_i\}_{i=1}^n\) be i.i.d. with density \(h_{\varepsilon}(\cdot)\) and distribution function \(H_{\varepsilon}(\cdot)\), and define \(\gamma = \mathbb{E}\bigl[\lvert \varepsilon_i\rvert\bigr]\). There exist constants \(b_0, b_1 > 0\) (possibly depending on \(\gamma\)) such that
\[
\begin{aligned}
& h_{\varepsilon}(x) \;\ge\; \tfrac{1}{b_0},
&& \text{for all } \lvert x\rvert \le 8\,\bigl(\tfrac{c_u}{c_l}\bigr)^{\tfrac12}\,\gamma,
\\[6pt]
& h_{\varepsilon}(x) \;\le\; \tfrac{1}{b_1},
&& \text{for all } x \in \mathbb{R}.
\end{aligned}
\]
\end{assumption}

\noindent
\textbf{Remark:} The lower bound in the density is essentially a (local) Bernstein condition \cite{alquier2019estimation} and is easily satisfied by many heavy-tailed distributions (e.g., any \(t\)-distribution with degrees of freedom \(v > 2\)). The upper bound simply requires that the noise distribution does not have unbounded peaks, which is also satisfied by common distributions such as Gaussian, Laplace, and \(t\)-distributions with \(v > 2\). As a result, Assumption~\ref{assumption3} is quite relaxed for a wide range of heavy-tailed settings. 

Under Assumption~\ref{assumption3}, in $t$-th iteration, the sub-gradient of the loss function 
\begin{align*}
\mathbf{G}_t 
&= \sum_{i\in \tilde{D}_t} \operatorname{sign}\Bigl(x_i^\top \boldsymbol{\beta}^t - y_i\Bigr)\tilde{x}_i \\
&= \sum_{i\in \tilde{D}_t} \operatorname{sign}\bigl(x_i^\top (\boldsymbol{\beta}^t - \boldsymbol{\beta}^*)\bigr) \tilde{x}_i 
   - \sum_{i\in \tilde{D}_t} \operatorname{sign}\bigl(x_i^\top \varepsilon_i\bigr) \tilde{x}_i
\end{align*}
 exhibits two distinct regimes based on the magnitude of \(\|\boldsymbol{\beta}^t - \boldsymbol{\beta}^*\|_2\). These regimes determine the behavior and convergence rate of the algorithm:

\begin{itemize}
    \item \textbf{Large Deviation Regime:} When \(\|\boldsymbol{\beta}^t - \boldsymbol{\beta}^*\|_2\) is relatively large, the sub-gradient \(\mathbf{G}_t\) is dominated by the term \(\sum_{i\in \tilde{D}_t} \operatorname{sign}\bigl(x_i^\top (\boldsymbol{\beta}^t - \boldsymbol{\beta}^*)\bigr) \tilde{x}_i\) leading to a larger error bound of $\mathbf{G}_t$. Hence the updates in this phase are large, allowing the algorithm to converge rapidly during the early iterations. This ensures efficient progress toward reducing the parameter error.

    \item \textbf{Small Deviation Regime:} Once \(\|\boldsymbol{\beta}^t - \boldsymbol{\beta}^*\|_2\) becomes small, the sub-gradient \(\mathbf{G}_t\) is primarily influenced by the noise term \(\sum_{i\in \tilde{D}_t} \operatorname{sign}\bigl(x_i^\top \varepsilon_i\bigr) \tilde{x}_i\). In this case the error bound of \(\mathbf{G}_t\) is small leading to a slower convergence rate that ensures refinement near the underlying parameter \(\boldsymbol{\beta}^*\).

\end{itemize}

Based on these differing bounds, we can establish the following convergence result.

\begin{algorithm}[tb]
\caption{Differentially Private Iterative Hard Thresholding with General Loss (DP-IHT-L)}
\label{alg:dp-iht-l}
\textbf{Input}: \(n\)-size dataset \(\mathbf{D}=\{(y_i, x_i)\}_{i=1}^n\), Step size \(\eta_t\), Sparsity level \(s\), Privacy parameters \(\varepsilon, \delta\), Truncation parameter \(K\), Total steps \(T\).

\textbf{Output}: \(\boldsymbol{\beta}^T\)

\begin{algorithmic}[1]
    \STATE \textbf{Initialization:} Set \(\boldsymbol{\beta}^0 = 0\).
    \STATE \textbf{Clipping:} For each \(i \in [n]\), define the truncated sample \(\tilde{x}_i \in \mathbb{R}^d\) by
    \[
    \tilde{x}_{i,j} = \operatorname{sign}(x_{i,j}) \min\!\Bigl\{\lvert x_{i,j} \rvert, K \Bigr\},\quad \forall j \in [d].
    \]
    \STATE Denote the truncated dataset as \(\tilde{D}=\{(\tilde{x}_i, y_i)\}_{i=1}^n\).
    \STATE \textbf{Data Splitting:} Split \(\tilde{D}\) into \(T\) disjoint subsets \(\{\tilde{D}_t\}_{t=0}^{T-1}\), each containing \(m = \frac{n}{T}\) samples.
    \FOR{\(t = 0\) to \(T-1\)}
        \STATE \(\displaystyle \boldsymbol{\beta}^{t+0.5} 
        = \boldsymbol{\beta}^t - \frac{\eta_t}{m} \sum_{i\in \tilde{D}_t} \operatorname{sign}\Bigl(x_i^\top \boldsymbol{\beta}^t - y_i\Bigr)\tilde{x}_i\)
        \STATE \(\displaystyle \boldsymbol{\beta}^{t+1} 
        = \Pi_{L}\Bigl(\text{Peeling}\Bigl(\boldsymbol{\beta}^{t+0.5}, \tilde{D}_t, s, \varepsilon, \delta, \tfrac{2\eta K}{m}\Bigr)\Bigr)\)
    \ENDFOR
    \STATE \textbf{return} \(\boldsymbol{\beta}^T\)
\end{algorithmic}
\end{algorithm}

\begin{theorem}\label{thm:two-phase}
Under Assumption \ref{assumption1}, \ref{assumption2} and \ref{assumption3} and assume \(n \geq  O\bigl(c_u\,c_l^{-1} s^* \log d\bigr)\). Set \(s = \Omega\bigl((c_u / c_l)^8 (b_0 / b_1)^8 s^*\bigr)\), $K=O(\log d)$, and the initial step size \(\eta_0\) satisfying
\[
\left[
\tfrac{c_l^{1/2}\,\|\boldsymbol{\beta}_0 - \boldsymbol{\beta}^*\|_2}{8\,n\,c_u}, \,
\tfrac{3c_l^{1/2}\,\|\boldsymbol{\beta}_0 - \boldsymbol{\beta}^*\|_2}{8\,n\,c_u}
\right],
\] 
then for the sequence \(\{\boldsymbol{\beta}^t\}\) generated by Algorithm~\ref{alg:dp-iht-l}, there are two distinct convergence phases: 
\begin{description}
    \item[\textbf{Phase One:}] 
    When \(\|\boldsymbol{\beta}^t - \boldsymbol{\beta}^*\|_2 \ge 8\,c_l^{-1/2}\,\gamma\), using \(\eta_t = (1 - c_1)^t \eta_0\) with \(c_1 = O(c_l\,c_u^{-1})\) ensures
    \[
    \|\boldsymbol{\beta}^{t+1} - \boldsymbol{\beta}^*\|_2 
    \;\;\le\;\; 
    (1 - c_1)^{t+1}\|\boldsymbol{\beta}_0 - \boldsymbol{\beta}^*\|_2 
    \;+\; 
     O\bigl(\sqrt{\boldsymbol{W}}\bigr)\bigr). 
    \]
    
    \item[\textbf{Phase Two:}] 
    Once \(\|\boldsymbol{\beta}^t - \boldsymbol{\beta}^*\|_2 \le 8\,c_l^{-1/2}\,\gamma\), switching to a constant step size \(\eta_t = O\bigl(c_l^{1/2} b_1^2 (n b_0 c_u)^{-1}\bigr)\) yields
    \[
    \|\boldsymbol{\beta}^{t+1} - \boldsymbol{\beta}^*\|_2 
    \;\le\;
    (1 - c_2^*)\,\|\boldsymbol{\beta}^t - \boldsymbol{\beta}^*\|_2 
    \;+\;
     O\bigl(\sqrt{\boldsymbol{W}}\bigr)\bigr).
    \]
\end{description}
Here, 
\[
 O\bigl(\sqrt{\boldsymbol{W}}\bigr)\bigr)
\;=\;
O\Bigl(\tfrac{T \bigl(s^*\bigr)^{3/2}\,\log d\,\bigl(\log(1/\delta)\bigr)^{1/2}\,\log (T/n)}{n\,\varepsilon}\Bigr)
\]
represents the cumulative effect of the Laplace noise, and \(c_2^* \in (0,1)\) indicates a strict contraction rate once \(\boldsymbol{\beta}^t\) is sufficiently close to the true parameter vector \(\boldsymbol{\beta}^*\).
\end{theorem}

\noindent
Because the sub-gradient operates under two different regimes, the overall estimation error follows a two-phase pattern. In Phase~One, the sub-gradient is governed by the smoothness property, leading to rapid convergence from larger errors. In Phase~Two, the strong convexity takes over once the estimator is sufficiently close to \(\boldsymbol{\beta}^*\), and the convergence becomes linear in nature. The \( O\bigl(\sqrt{\boldsymbol{W}}\bigr)\bigr)\) term arises from comparing against the exact parameter \(\boldsymbol{\beta}^*\), which is not affected by the additional noise injected via the peeling procedure. With these two phases established, we can derive a global error bound as follows.

\begin{theorem}\label{thm:error-bound}
Consider the same settings as in Theorem \ref{thm:two-phase}, with probability at least 
\(
1 - \exp\Bigl(-\,C\,s^*\,\log \Bigl(\tfrac{2\,d}{s^*}\Bigr)\Bigr),
\)
after at most
\[
T=O\Bigl(\log \Bigl(\tfrac{\|\boldsymbol{\beta}_0 - \boldsymbol{\beta}^*\|_2}{\gamma}\Bigr) 
\;+\; 
\log \Bigl(n\,\gamma\,b_0^{-1}\,\log\bigl(\tfrac{2\,d}{s^*}\bigr)\Bigr)\Bigr)
\]
iterations, Algorithm~\ref{alg:dp-iht-l} produces an estimator \(\boldsymbol{\beta}^T\) satisfying
\[
\|\boldsymbol{\beta}^T - \boldsymbol{\beta}^*\|_2 
\;\le\; 
O\Bigl(\frac{(s^*)^{3/2}\,\log d \;\bigl(\log(\tfrac{1}{\delta})\bigr)^{1/2}\,\log n}{n\,\varepsilon}\Bigr).
\]
\end{theorem}
Overall, our DP-IHT-L algorithm achieves an error bound of  
\(
\tilde{O}\!\Bigl(\frac{(s^*)^{3/2} \log d}{n \varepsilon}\Bigr)
\)
for high-dimensional sparse data with heavy-tailed distributions. In comparison, the DP-IHT-H algorithm attains an error of  
\(
\tilde{O}\!\Bigl(s^{*\frac{1+2\zeta}{2+2\zeta}} \Bigl(\frac{\log^2 d}{n \varepsilon}\Bigr)^{\frac{\zeta}{1+\zeta}}\Bigr),
\)
which is larger than our previous bound. Similarly, when \(\zeta=1\) , our error bound is almost the same as in \cite{hu2022high}, which is given by  
\(
\tilde{O}\!\Bigl(s^{*\frac{1+2\zeta}{1+\zeta}} \Bigl(\frac{\log^3 n \log^2 d}{n^2 \varepsilon^2}\Bigr)^{\frac{\zeta}{1+\zeta}}\Bigr).
\) Note that in the case of  \(\zeta=1\), \cite{cai2021cost} achieves an error bound of  
\(
\tilde{O}\!\Bigl(\sqrt{\frac{s^* \log d}{n}}+\frac{s^* \log d}{n \varepsilon}\Bigr). 
\)
However, their analysis requires the strong condition that  
\(
\|x_I\|_\infty \leq O\!\Bigl(\frac{1}{\sqrt{|I|}}\Bigr)
\)
for any index set \(I \subseteq [d]\). Thus, their results is imcomparable with ours. 
\section{Experiments}
\label{sec:experiments}
In this section, we evaluate the practical performance of our proposed algorithms on both synthetic and real-world datasets.
\subsection{Experimental Setup}
\label{subsec:data_generation}

\paragraph{Synthetic Data.}
We generate synthetic data for linear regression following the model
\[
    y_i = \langle x_i, \boldsymbol{\beta}^* \rangle + \varepsilon_i, \quad i = 1, \ldots, n,
\]
where each feature vector \(x_i \in \mathbb{R}^d\) is drawn from \(\mathcal{N}(\mathbf{0}, \mathbf{I}_d)\). The true coefficient vector \(\boldsymbol{\beta}^*\) has \(s^*\) nonzero entries, which are sampled from a scaled standard normal distribution, with the remaining entries set to zero. To simulate heavy-tailed noise, the error terms \(\{\varepsilon_i\}\) are drawn from a Student-\(t\) distribution. Specifically, we set the degrees of freedom \(\nu\) to 1.75 when \(\zeta = 0.5\) and to 3 when \(\zeta = 1\).

\paragraph{Real-World Data.}
For additional validation, we evaluate our algorithms on real-world datasets, including: \textbf{NCI-60 cancer cell line dataset}~\cite{Reinhold2012}, with \(n = 59\) samples and \(d = 14{,}342\) features. And \textbf{METABRIC} (Molecular Taxonomy of Breast Cancer International Consortium) dataset~\cite{Curtis2012,Pereira2016}, with \(n = 1{,}904\) samples and \(d = 24{,}368\) features. These datasets, commonly found in publicly available biological databases, are known to exhibit heavy-tailed distributions.

\paragraph{Parameter Choices.}  
Unless otherwise specified, the default parameter settings for the DP-IHT-H algorithm are as follows: \(s^* = 5\), \(\epsilon = 0.5\), \(\zeta = 1\), \(\delta = 1 / n^{1.1}\), and the Huber loss parameter \(\tau = 1\). For DP-IHT-L, the same parameters are applied. The truncation parameter \(K\) is set to \(\log d\). Across all algorithms, we use a constant step size of \(\eta = 0.01\).

\paragraph{Evaluation Metrics.}
We evaluate algorithm performance using the \(\ell_2\)-estimation error, defined as \(\|\hat{\boldsymbol{\beta}} - \boldsymbol{\beta}^*\|_2\). As there is no underlying parameter $\boldsymbol{\beta}^*$ for real-world data, we will use the adaHuher algorithm in~\cite{sun2020adaptive}, which achieves the optimal rate, to approximate $\boldsymbol{\beta}^*$.  Each experiment is repeated 20 times, and we report the average results to ensure statistical reliability.

\paragraph{Baselines.} 
We focus on the DP-IHT-H and DP-IHT-L algorithms. As we mentioned above, there is no previous research on the DP sparse model with heavy-tailed response that only has the $1+\zeta$-th moment with $\zeta\in (0, 1)$. For comparison, we include:
\begin{itemize}
    \item \textbf{DP-SLR}: Differentially private sparse linear regression under regular (light-tailed) data~\cite{cai2021cost}. DP-SLR could achieve the almost optimal rate in this setting. 
    \item \textbf{adaHuber}: Non-DP linear regression in high-dimensional sparse heavy-tailed settings using the adaptive Huber loss method~\cite{sun2020adaptive}. adaHuber achieves the optimal rate in the non-private setting. 
\end{itemize}

\subsection{Results on Synthetic Data}
\label{subsec:results_synthetic_data}
We first investigate whether DP-IHT-H achieves better performance under heavy tails (i.e., small \(\zeta\)), whether it outperforms differentially private algorithms designed for regular data distributions, and the performance gap between DP-IHT-H and the non-private optimal  algorithm.

\begin{figure}[t]
    \centering
    \begin{subfigure}[b]{0.45\textwidth}
        \includegraphics[width=\textwidth]{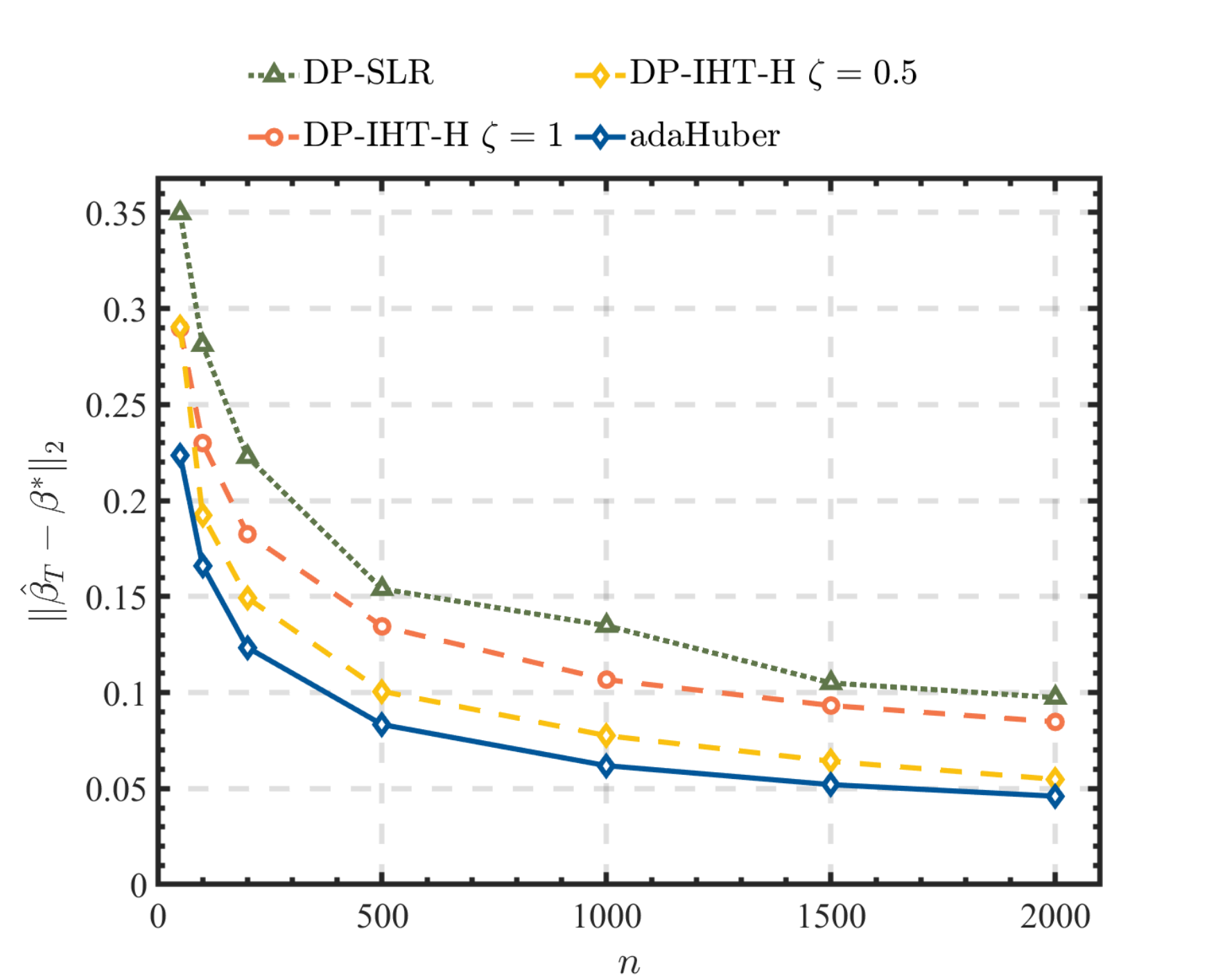}
        \caption{Comparison for different \(\zeta\)}
        \label{1.a}
    \end{subfigure}
    \hfill
    \begin{subfigure}[b]{0.45\textwidth}
        \includegraphics[width=\textwidth]{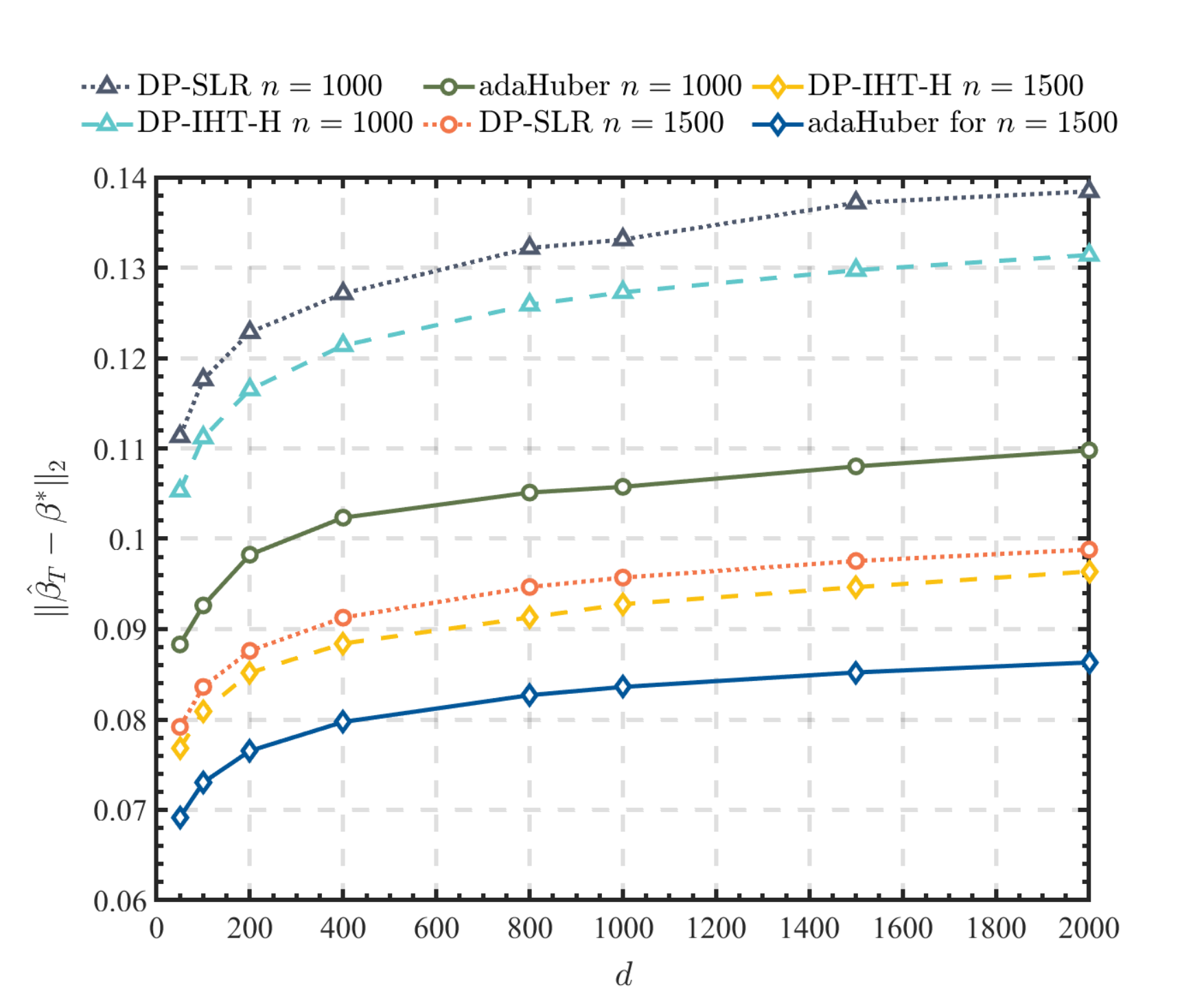}
        \caption{Performance vs.\ \(d\)}
        \label{1.b}
    \end{subfigure}
    
    \begin{subfigure}[b]{0.45\textwidth}
        \includegraphics[width=\textwidth]{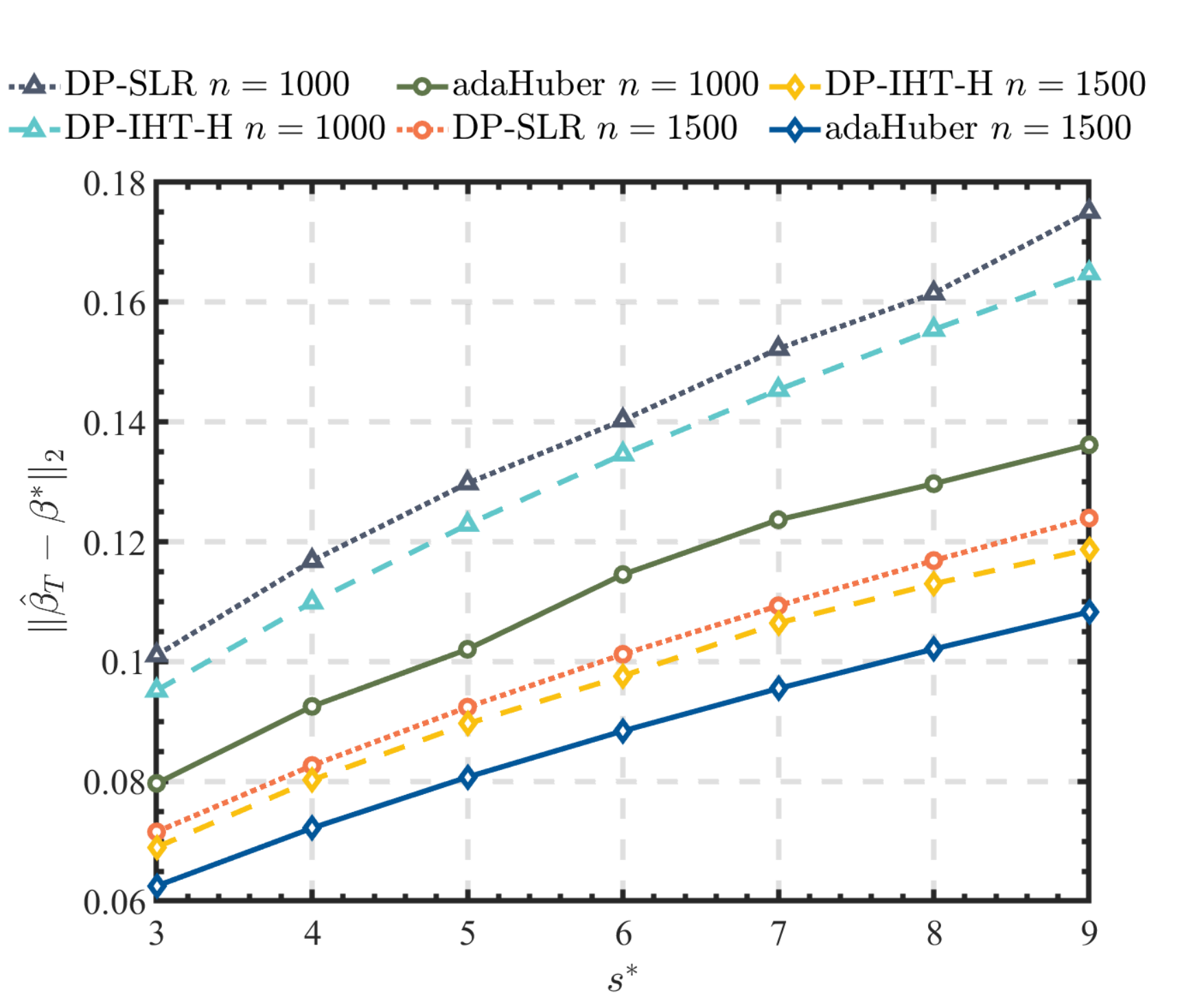}
        \caption{Performance vs.\ \(s^*\)}
        \label{1.c}
    \end{subfigure}
    \hfill
    \begin{subfigure}[b]{0.45\textwidth}
        \includegraphics[width=\textwidth]{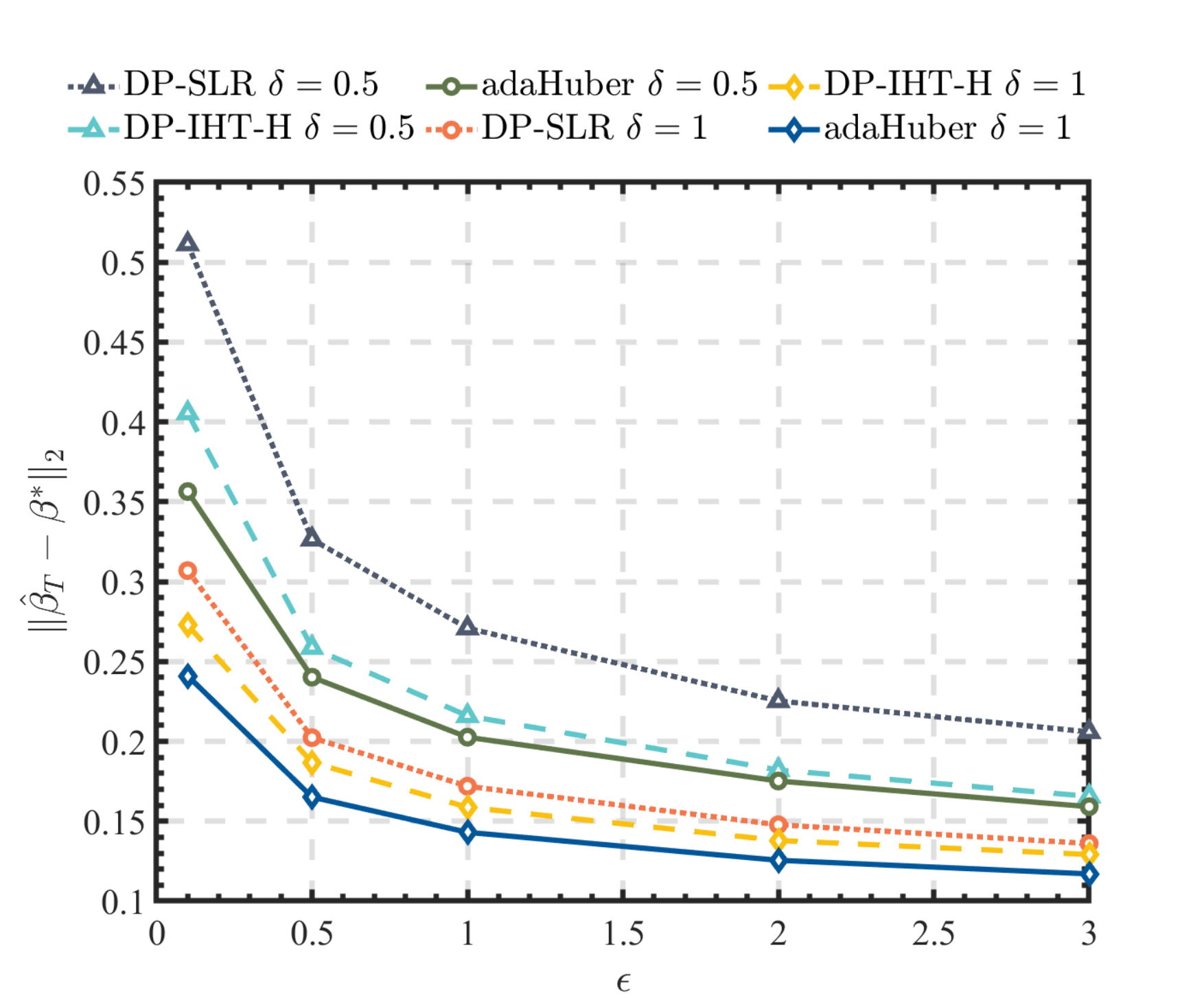}
        \caption{Performance vs.\ \(\epsilon\)}
        \label{1.d}
    \end{subfigure}
    \caption{Comparison of DP-IHT-H, DP-SLR, and adaHuber methods.}
\end{figure}

Figure~\ref{1.a} shows that when the response variable is heavy-tailed, DP-IHT-H achieves significantly lower errors than DP-SLR, possibly because of the square loss used in DP-SLR is less robust to outliers. Moreover, Figure~\ref{1.b} illustrates that as the dimensionality \(d\) increases, the estimation error grows more gradually for DP-IHT-H than for the other two methods. In addition, the figure clearly indicates that DP-IHT-H outperforms the DP-SLR algorithm across various values of \(d\). Figures~\ref{1.c} and \ref{1.d} further confirm that in heavy-tailed scenarios, DP-IHT-H consistently outperforms the DP-SLR algorithm. Across varying sparsity \(s^*\) and privacy budgets \(\epsilon\), DP-IHT-H maintains robust performance, demonstrating its advantage in handling heavy-tailed data while preserving differential privacy.

We next evaluate the DP-IHT-L algorithm in comparison with DP-IHT-H, focusing on whether DP-IHT-L provides a similar error guarantee for different values of \(\zeta\), and whether it outperforms DP-IHT-H for larger \(\zeta\). From Figures~\ref{fig:2a} and \ref{fig:2b}, we observe that DP-IHT-L generally outperforms DP-IHT-H, particularly when \(\zeta\) is smaller. Figure~\ref{fig:2b} also indicates that when \(\zeta = 1\), the difference between the two methods is minimal. As \(\zeta\) decreases, the performance of DP-IHT-L remains relatively stable. Figure~\ref{fig:2c} shows that DP-IHT-H can perform better with respect to \(s^*\), consistent with its dependence of \(O(s^{*3/4})\) (in contrast to \(O(s^{*3/2})\) for DP-IHT-L). Finally, Figure~\ref{fig:2d} demonstrates that for \(\zeta = 1\), both algorithms has less estimation error as $\epsilon$ decrease.

\begin{figure}[t]

    \centering
    \begin{subfigure}[b]{0.45\textwidth}
        \includegraphics[width=\textwidth]{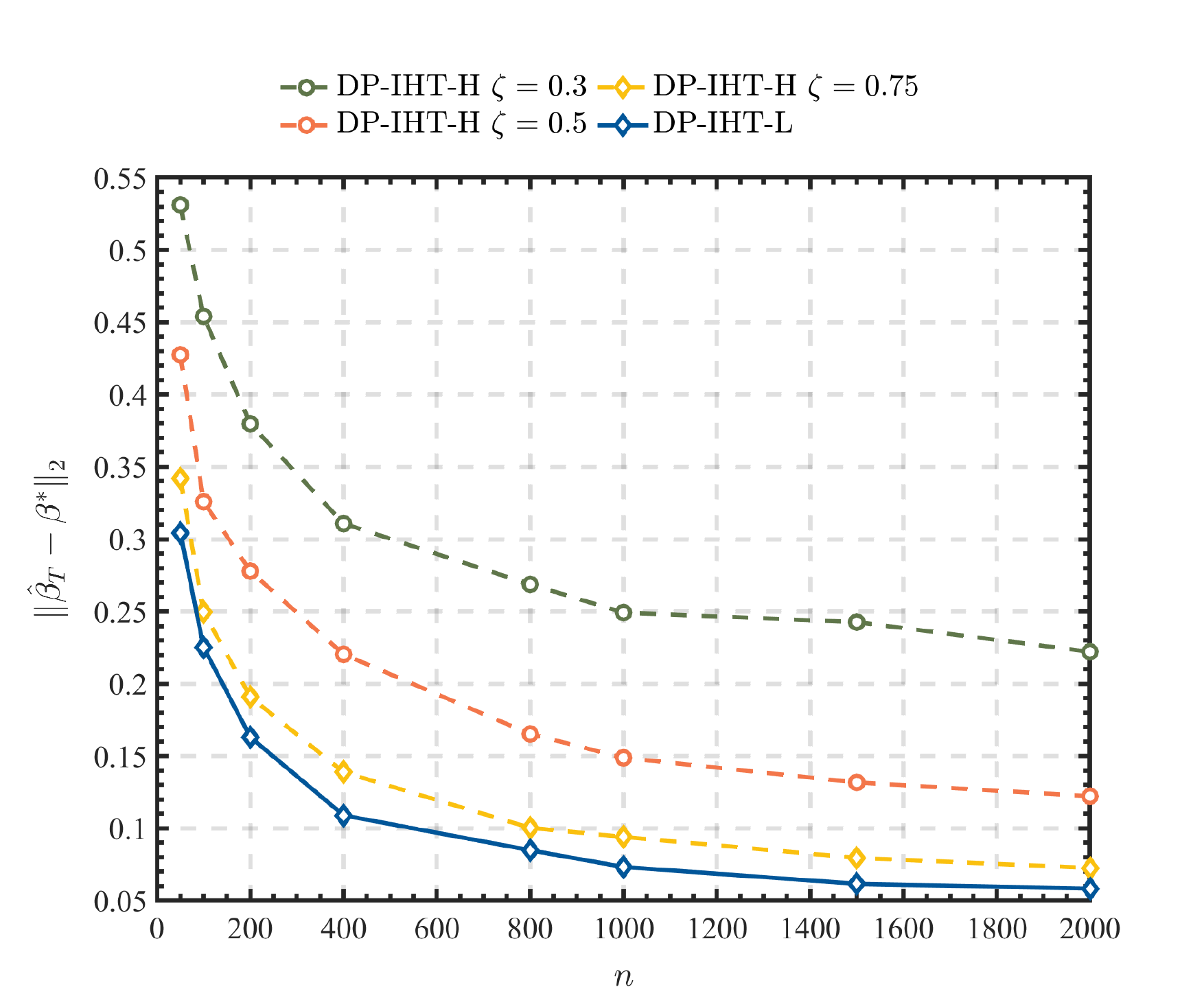}
        \caption{Comparison for different \(\zeta\)}
        \label{fig:2a}
    \end{subfigure}
    \hfill
    \begin{subfigure}[b]{0.45\textwidth}
        \includegraphics[width=\textwidth]{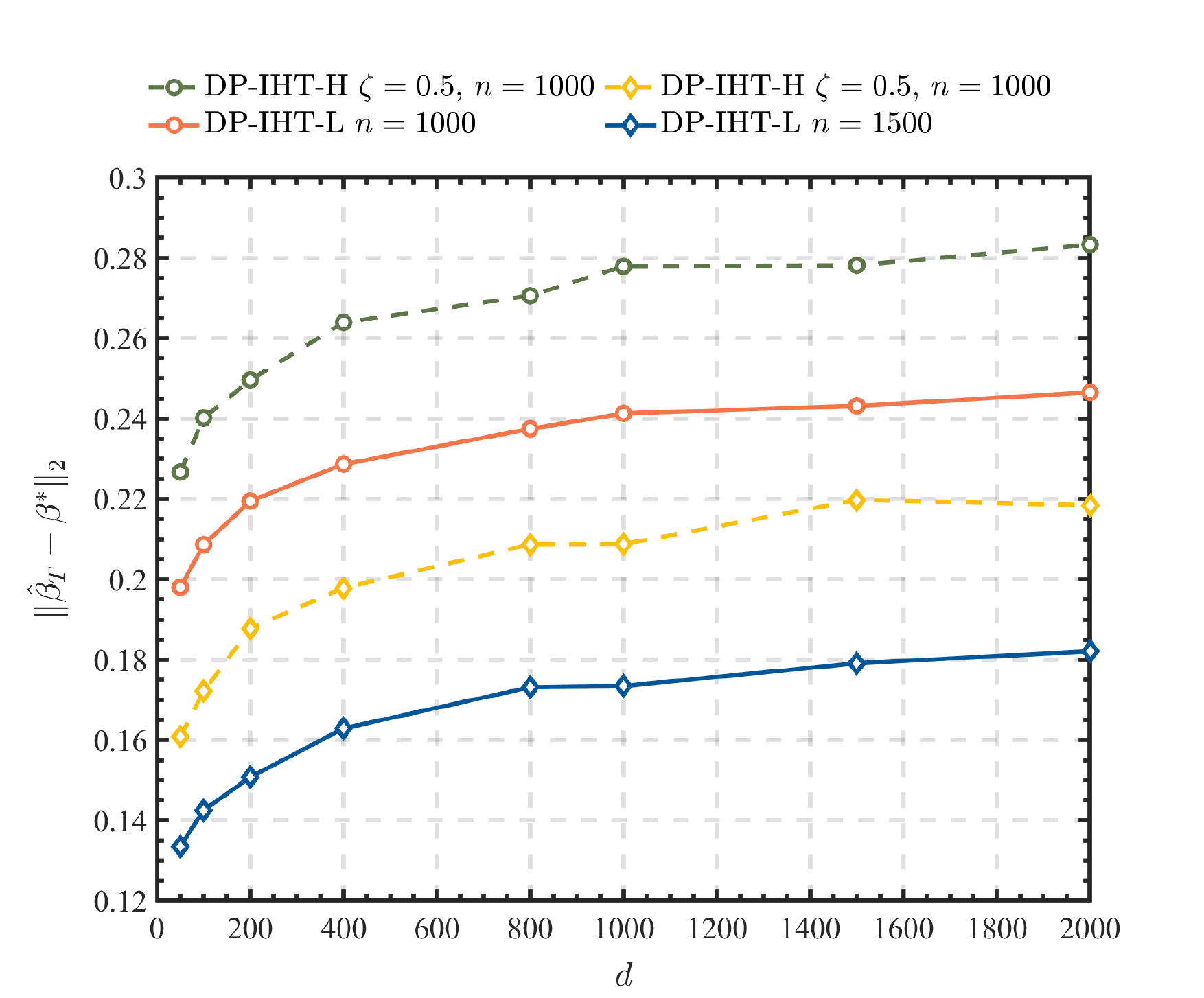}
        \caption{Performance vs.\ \(d\)}
        \label{fig:2b}
    \end{subfigure}


    \begin{subfigure}[b]{0.45\textwidth}
        \includegraphics[width=\textwidth]{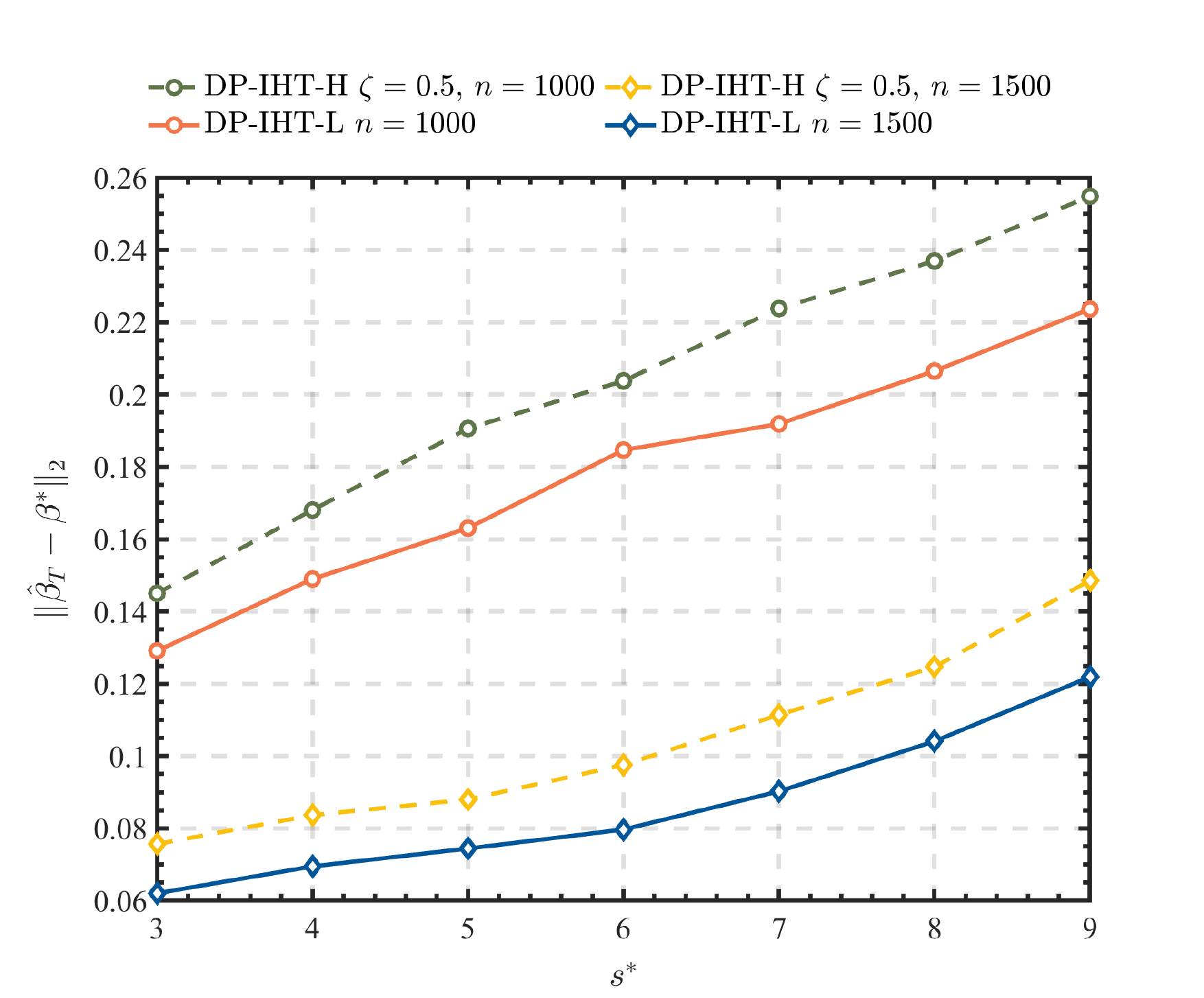}
        \caption{Performance vs.\ \(s^*\)}
        \label{fig:2c}
    \end{subfigure}
    \hfill
    \begin{subfigure}[b]{0.45\textwidth}
        \includegraphics[width=\textwidth]{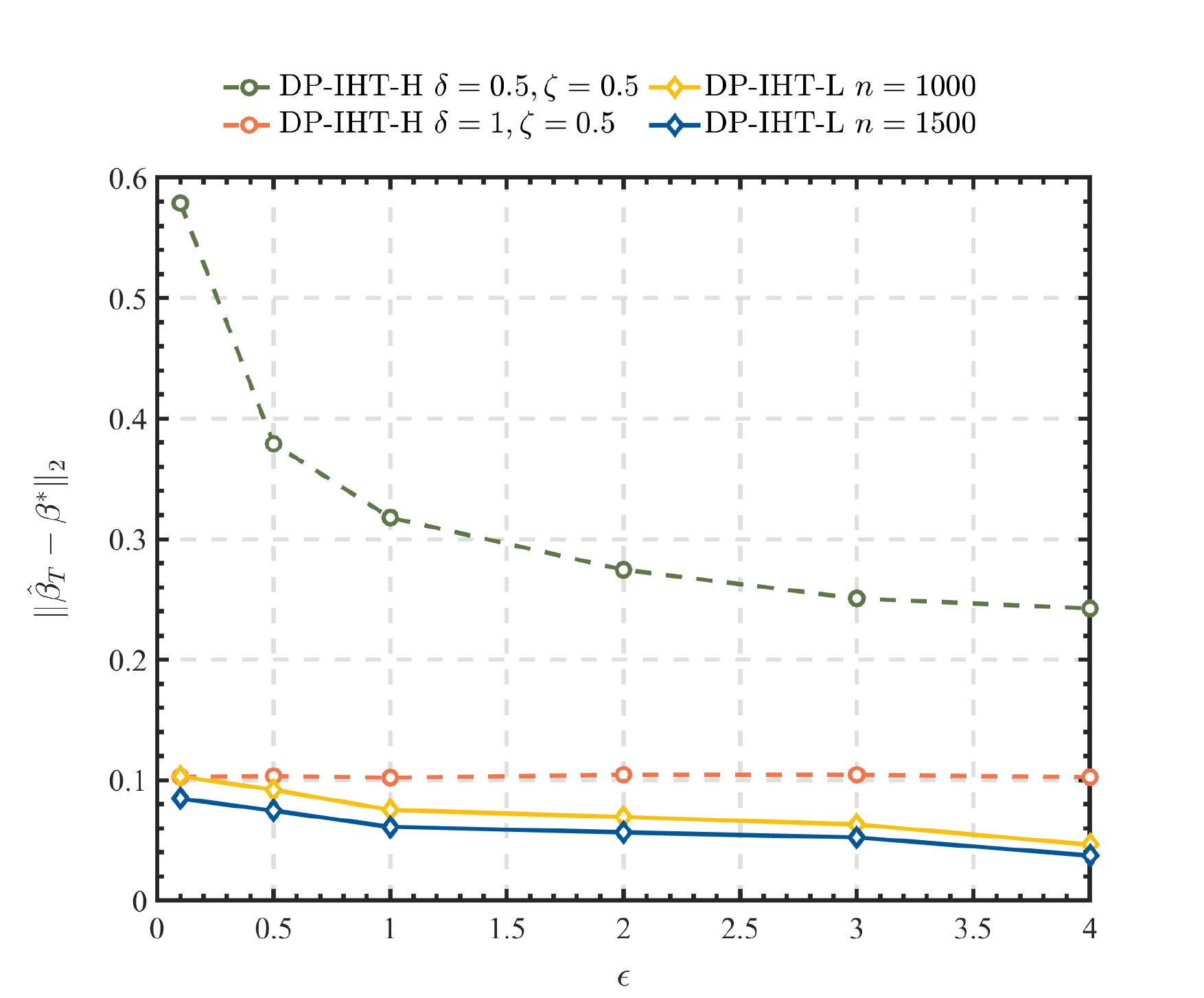}
        \caption{Performance vs.\ \(\epsilon\)}
        \label{fig:2d}
    \end{subfigure}
    \caption{Comparison of DP-IHT-L and DP-IHT-H across various metrics.  }
    \label{fig:comparison}
\end{figure}

Overall, these results indicate that DP-IHT-H and DP-IHT-L are particularly well suited for highly heavy-tailed data, while DP-IHT-L may be preferred in more moderate scenarios due to its stability. 

\subsection{Real Data Analysis}
\label{subsec:real_data_analysis}
We evaluate four methods—adaHuber, DP-SLR, DP-IHT-H, and DP-IHT-L—on two genomic datasets to assess their performance in privacy-preserving robust estimation.

\paragraph{NCI-60 Dataset.}
Following the protocols in \cite{Reinhold2012,sun2020adaptive}, we analyze a dataset of protein expression (from 162 antibodies) and RNA transcript levels across 60 cancer cell lines to identify genes affecting \emph{KRT19} expression \cite{Nakata2004}. After preprocessing, the dataset comprises \(n=59\) samples and \(d=14\,342\) features. For the DP methods, we set \(s^* = 5\), \(\varepsilon = 0.5\), and \(\delta = 1/n^{1.1}\). Table~\ref{tab:real_data_results} summarizes the results.

\begin{table}[h!]
\centering
\caption{Results on the NCI-60 dataset.}
\label{tab:real_data_results}
\renewcommand{\arraystretch}{1.2}
\setlength{\tabcolsep}{4pt}
\small
\begin{tabularx}{\textwidth}{cccc}
\toprule
\textbf{Method} & \textbf{MAE} & \textbf{Size} & \textbf{Selected Genes} \\
\midrule
adaHuber   & 2.07 & 5 & MALL, TM4SF4, ANXA3, ADRB2, NRN1 \\
DP-SLR     & 2.72 & 5 & MALL, TGFBI, S100A6, LPXN, DSP \\
DP-IHT-H   & 2.40 & 5 & MALL, ANXA3, NRN1, CA2, EPS8L2 \\
DP-IHT-L   & 2.34 & 5 & MALL, NRN1, DSP, AUTS2, EPS8L2 \\
\bottomrule
\end{tabularx}
\end{table}

The non-private adaHuber achieves the lowest MAE, while the DP-IHT variants perform competitively and surpass DP-SLR, highlighting the benefits of robust estimation in heavy-tailed data.

\paragraph{METABRIC Dataset.}
We further assess the methods on the METABRIC breast cancer dataset \cite{Curtis2012,Pereira2016}, which contains \(n = 1\,904\) samples and \(d = 24\,368\) features. Here, the parameters are set as \(s^* = 5\), \(\varepsilon = 1.0\), and \(\delta = 1/n^{1.1}\). Table~\ref{tab:metabric_gene_results} reports the results.

\begin{table}[h!]
\centering
\caption{Results on the METABRIC dataset.}
\label{tab:metabric_gene_results}
\renewcommand{\arraystretch}{1.2}
\setlength{\tabcolsep}{4pt}
\small
\begin{tabularx}{\textwidth}{cccc}
\toprule
\textbf{Method} & \textbf{MAE} & \textbf{Size} & \textbf{Selected Genes} \\
\midrule
adaHuber   & 0.92 & 5 & PIK3CA, MUC16, SYNE1, KMT2C, GATA3 \\
DP-SLR     & 1.22 & 5 & MUC16, CDH1, MAP3K1, NCOR2, CBFB \\
DP-IHT-H   & 1.08 & 5 & PIK3CA, MUC16, AHNAK2, MAP3K1, GATA3 \\
DP-IHT-L   & 1.05 & 5 & PIK3CA, SYNE1, NOTCH1, TG, KMT2C \\
\bottomrule
\end{tabularx}
\end{table}

As in the NCI-60 dataset, adaHuber attains the best MAE, and the DP-IHT methods perform comparably while outperforming DP-SLR, even in this lighter-tailed setting.

\section{Conclusion}
This work presented novel approaches for differentially private linear regression that address the challenges posed by heavy-tailed data distributions. We introduced two algorithms: DP-IHT-H and DP-IHT-L, each designed to handle different tail behaviors. DP-IHT-H demonstrates strong performance for moderately heavy-tailed data, achieving an optimized error bound under $(\epsilon, \delta)$-differential privacy. However, its performance degrades for heavier tails, prompting the development of DP-IHT-L, which achieves stable error bounds irrespective of the tail behavior. Extensive experiments on synthetic and real-world datasets verified the effectiveness of our methods, showing their robustness and applicability in diverse scenarios.

\section*{Acknowledgements}
This work is supported in part by the funding BAS/1/1689-01-01, URF/1/4663-01-01, REI/1/5232-01-01, REI/1/5332-01-01, and URF/1/5508-01-01 from KAUST, and funding from KAUST - Center of Excellence for Generative AI, under award number 5940.


\newpage
\appendix
\section{Proof of Theorem~\ref{thm:dp}}\label{proof:dp}
\begin{proof}

Consider two neighboring datasets \(D\) and \(D'\) that differ by exactly one record. At each iteration of DP-IHT-H, we update
\[
\boldsymbol{\beta}^{t+0.5} 
\;=\; 
\boldsymbol{\beta}^t 
\;-\; 
\eta\,\nabla \mathcal{L}_\tau(\boldsymbol{\beta}^t).
\]
Within both the quadratic and absolute-value regimes of the Huber function, the gradient satisfies \(\|\nabla \mathcal{L}_\tau(\boldsymbol{\beta}^t)\|_{\infty} \le \tau K\) because of clipping each feature $x_i$ to $\left\|\tilde{x}_i\right\|_{\infty} \leq K$. Hence,
\begin{align*}
\bigl\|\boldsymbol{\beta}^{t+0.5} - \boldsymbol{\beta}^{\prime t+0.5}\bigr\|_\infty 
&= \bigl\|-\eta\nabla \mathcal{L}_\tau(\boldsymbol{\beta}^t) 
+ \eta\nabla \mathcal{L}_\tau(\boldsymbol{\beta}^{\prime t})\bigr\|_\infty \\[6pt]
&\leq \eta \,\bigl\|\nabla \mathcal{L}_\tau(\boldsymbol{\beta}^t) 
- \nabla \mathcal{L}_\tau(\boldsymbol{\beta}^{\prime t})\bigr\|_\infty.
\end{align*}
Because \(D\) and \(D'\) differ by only one sample, the change in their gradients is on the order of \(\tfrac{\tau K}{m}\). Therefore,
\[
\bigl\|\boldsymbol{\beta}^{t+0.5} - \boldsymbol{\beta}^{\prime t+0.5}\bigr\|_\infty
\;\le\;
\eta \cdot \frac{\tau K}{m}.
\]
By applying the composition property of differential privacy (``Peeling'' is \((\varepsilon,\delta)\)-DP was proved in lemma 3.3 of \cite{cai2021cost}) and noting that each iteration's update is influenced by only one sample to a bounded extent, we conclude that every step is \((\varepsilon, \delta)\)-DP. Consequently, the entire procedure maintains \((\varepsilon,\delta)\)-DP.
\end{proof}
\section{Lemmas for Theorem \ref{theorem2}}
There are two main preparations for Theorem \ref{theorem2}. The first part relates to the properties of Restricted Strong Convexity (RSC) and Restricted Strong Smoothness (RSS). The second part concerns the properties of the ``Peeling" algorithm.

We start by discussing key properties related to RSC and RSS.

\begin{lemma}[\textbf{Restricted Strong Convexity}]
Under Assumption \ref{assumption1} and \ref{assumption2} for \(\tau \geq 2 \max\{(4v_\zeta)^{\frac{1}{1+\zeta}}, 4A_1^2r\}\) and \(n \gtrsim (\tau/r)^2(d+t)\), with probability at least \(1 - e^{-t}\), we have:
\begin{equation*}
    \langle\nabla\mathcal{L}_\tau(\boldsymbol{\beta})-\nabla \mathcal{L}_\tau(\boldsymbol{\beta}^*),\boldsymbol{\beta}-\boldsymbol{\beta}^*\rangle \geq c_l\|\boldsymbol{\beta}-\boldsymbol{\beta}^*\|^2_{\boldsymbol{\Sigma}, 2}
\end{equation*}
uniformly over \(\boldsymbol{\beta} \in \Theta_0(r) = \{\boldsymbol{\beta} \in \mathbb{R}^d : \|\boldsymbol{\beta} - \boldsymbol{\beta}^*\|_{\boldsymbol{\Sigma}, 2} \leq r\}\).
\end{lemma}

\begin{proof}
    The proof of this lemma follows directly from Lemma 4 in \cite{sun2020adaptive}. 
\end{proof}

\begin{lemma}[\textbf{Restricted Strong Smoothness}]
Under Assumption \ref{assumption1} and \ref{assumption2} with \(0 < c_l < c_u\), we have:
\begin{equation*}
    \langle\nabla\mathcal{L}_\tau(\boldsymbol{\beta})-\nabla \mathcal{L}_\tau(\boldsymbol{\beta}^*),\boldsymbol{\beta}-\boldsymbol{\beta}^*\rangle \leq c_u\|\boldsymbol{\beta}-\boldsymbol{\beta}^*\|^2_{\boldsymbol{\Sigma}, 2}.
\end{equation*}
\end{lemma}

\begin{proof}
By the mean value property:
\begin{equation*}
    \|\nabla\mathcal{L}_\tau(\boldsymbol{\beta})-\nabla \mathcal{L}_\tau(\boldsymbol{\beta}^*)\|_2 = \|\nabla^2L_\tau(\bar{\boldsymbol{\beta}})(\boldsymbol{\beta}-\boldsymbol{\beta}^*)\|_2,
\end{equation*}
where \(\bar{\boldsymbol{\beta}}\) is a point on the line segment between \(\boldsymbol{\beta}\) and \(\boldsymbol{\beta}^*\).

From Assumption \ref{assumption2}, we know:
\begin{align*}
    \langle u, \nabla^2L_\tau(\bar{\boldsymbol{\beta}})u \rangle &= \langle u, \frac{1}{n}\sum_{i=1}^n x_i x_i^\top 1(|y_i - x_i^\top \bar{\boldsymbol{\beta}}| \leq \tau) u \rangle \\
    &\leq u^\top S_n u \leq c_u u^\top u,
\end{align*}
where \(S_n\) is the empirical covariance matrix.

Since \(\nabla^2L_\tau(\bar{\boldsymbol{\beta}})\) is positive definite, we also have:
\begin{equation*}
    \|\nabla^2L_\tau(\bar{\boldsymbol{\beta}})^\top u\|_2^2 \leq c_u^2 \|u\|_2^2 \quad \Rightarrow \quad \|\nabla\mathcal{L}_\tau(\boldsymbol{\beta})-\nabla \mathcal{L}_\tau(\boldsymbol{\beta}^*)\|^2_2 \leq c_u^2\|\boldsymbol{\beta}-\boldsymbol{\beta}^*\|^2_{\boldsymbol{\Sigma}, 2}.
\end{equation*}

Thus:
\begin{equation*}
    \langle\nabla\mathcal{L}_\tau(\boldsymbol{\beta})-\nabla \mathcal{L}_\tau(\boldsymbol{\beta}^*),\boldsymbol{\beta}-\boldsymbol{\beta}^*\rangle \leq \|\nabla\mathcal{L}_\tau(\boldsymbol{\beta})-\nabla \mathcal{L}_\tau(\boldsymbol{\beta}^*)\|_2 \|\boldsymbol{\beta}-\boldsymbol{\beta}^*\|_2 \leq c_u\|\boldsymbol{\beta}-\boldsymbol{\beta}^*\|^2_{\boldsymbol{\Sigma}, 2}.
\end{equation*}
\end{proof}

\begin{lemma}\label{newl}
Under assumption \ref{assumption1} and \ref{assumption2} for \(0 < \delta \leq 1\), with probability at least \(1 - 2e^{-t}\), we have:
\begin{equation*}
\| \nabla \mathcal{L}_\tau(\boldsymbol{\beta}^*)\|_\infty \lesssim 2L \sqrt{\frac{v_\zeta \tau^{1-\delta} t}{n}} + L \frac{\tau t}{2n} + Lv_\zeta \tau^{-\delta}.
\end{equation*}

By tuning \(\tau \asymp (n/t)^{\frac{1}{1+\zeta}}\), it also holds that:
\begin{equation*}
    \|\nabla \mathcal{L}_\tau(\boldsymbol{\beta}^*)\|_\infty \lesssim \tau L\frac{t}{n}.
\end{equation*}
\end{lemma}

\begin{proof}
Under Assumption \ref{assumption2} with \(0 < \delta < 1\), define \(\xi_i = \psi_\tau(\varepsilon_i)\), where \(\varepsilon_i\) is the regression error. Then:
\[
\nabla \mathcal{L}_\tau(\boldsymbol{\beta}^*) = -\frac{1}{n} \sum_{i=1}^n \xi_i \boldsymbol{x}_i.
\]
For each \(1 \leq j \leq d\):
\[
\left|\mathbb{E}(\xi_i x_{ij})\right| = |x_{ij}| \cdot \left|\mathbb{E}(\xi_i)\right| \leq L v_\zeta \tau^{-\delta}.
\]

Using Bernstein's inequality, with probability at least \(1 - 2e^{-t}\):
\[
\left|\frac{1}{n} \sum_{i=1}^n (\xi_i x_{ij} - \mathbb{E} \xi_i x_{ij})\right| \lesssim 2L \sqrt{\frac{v_\zeta \tau^{1-\delta} t}{n}} + L \frac{\tau t}{2n}.
\]

Combining these results gives:
\[
\| \nabla \mathcal{L}_\tau(\boldsymbol{\beta}^*)\|_\infty \lesssim 2L \sqrt{\frac{v_\zeta \tau^{1-\delta} t}{n}} + L \frac{\tau t}{2n} + Lv_\zeta \tau^{-\delta}.
\]
By setting \(\tau \asymp (n/t)^{\frac{1}{1+\zeta}}\), we further simplify:
\[
\|\nabla \mathcal{L}_\tau(\boldsymbol{\beta}^*)\|_\infty \lesssim \tau L \frac{t}{n}.
\]
\end{proof}

Then we introduce two key properties of the ``Peeling" algorithm (Algorithm \ref{alg:peeling}).

\begin{lemma}(\cite{cai2021cost})\label{A.31}
    Let $\tilde{P}_s$ be defined as in Algorithm 2. For any index set $I$, any $\boldsymbol{v} \in \mathbb{R}^I$ and $\hat{\boldsymbol{v}}$ such that $\|\hat{\boldsymbol{v}}\|_0 \leq \hat{s} \leq s$, we have that for every $c>0$,
$$
\left\|\tilde{P}_s(\boldsymbol{v})-\boldsymbol{v}\right\|_2^2 \leq(1+1 / c) \frac{|I|-s}{|I|-\hat{s}}\|\hat{\boldsymbol{v}}-\boldsymbol{v}\|_2^2+4(1+c) \sum_{i \in[\mathrm{s}]}\left\|\boldsymbol{w}_i\right\|_{\infty}^2
$$
\end{lemma}

\begin{proof}
 Let $\psi: R_2 \rightarrow R_1$ be a bijection. By the selection criterion of Algorithm \ref{alg:peeling}, for each $j \in R_2$ we have $\left|v_j\right|+w_{i j} \leq\left|v_{\psi(j)}\right|+w_{i \psi(j)}$, where $i$ is the index of the iteration in which $\psi(j)$ is appended to $S$. It follows that, for every $c>0$,

\begin{align*}
v_j^2 & \leq\left(\left|v_{\psi(j)}\right|+w_{i \psi(j)}-w_{i j}\right)^2 \\
& \leq(1+1 / c) v_{\psi(j)}^2+(1+c)\left(w_{i \psi(j)}-w_{i j}\right)^2 \leq(1+1 / c) v_{\psi(j)}^2+4(1+c)\left\|\boldsymbol{w}_i\right\|_{\infty}^2
\end{align*}

Summing over $j$ then leads to
$$
\left\|\boldsymbol{v}_{R_2}\right\|_2^2 \leq(1+1 / c)\left\|\boldsymbol{v}_{R_1}\right\|_2^2+4(1+c) \sum_{i \in[s]}\left\|\boldsymbol{w}_i\right\|_{\infty}^2
$$
\end{proof}

\begin{lemma}(\cite{cai2021cost})\label{A.3}
    Let $\bar{P}_s$ be defined as in Algorithm \ref{alg:peeling}. For any index set $I$, any $\boldsymbol{v} \in \mathbb{R}^I$ and $\hat{\boldsymbol{v}}$ such that $\|\hat{\boldsymbol{v}}\|_0 \leq \hat{s} \leq s$, we have that for every $c>0$,
$$
\left\|\tilde{P}_s(\boldsymbol{v})-\boldsymbol{v}\right\|_2^2 \leq(1+1 / c) \frac{|I|-s}{|I|-\hat{s}}\|\hat{\boldsymbol{v}}-\boldsymbol{v}\|_2^2+4(1+c) \sum_{i \in[s]}\left\|\boldsymbol{w}_i\right\|_{\infty}^2 .
$$
\end{lemma} 
\begin{proof}
    Let $T$ be the index set of the top $s$ coordinates of $v$ in terms of absolute values. We have

\begin{align*}
\left\|\tilde{P}_s(\boldsymbol{v})-\boldsymbol{v}\right\|_2^2 & =\sum_{j \in S^c} v_j^2=\sum_{j \in S^c \cap T^c} v_j^2+\sum_{j \in S^c \cap T} v_j^2 \\
& \leq \sum_{j \in S^c \cap T^c} v_j^2+(1+1 / c) \sum_{j \in S \cap T^c} v_j^2+4(1+c) \sum_{i \in[s]}\left\|\boldsymbol{w}_i\right\|_{\infty}^2 .
\end{align*}

The last step is true by observing that $\left|S \cap T^c\right|=\left|S^{\mathrm{c}} \cap T\right|$ and applying lemma \ref{A.31}.
Now, for an arbitrary $\hat{\boldsymbol{v}}$ with $\|\hat{\boldsymbol{v}}\|_0=\hat{s} \leq s$, let $\hat{S}=\operatorname{supp}(\hat{\boldsymbol{v}})$. We have
$$
\frac{1}{|I|-s} \sum_{j \in T^c} v_j^2=\frac{1}{\left|T^c\right|} \sum_{j \in T^c} v_j^2 \stackrel{(*)}{\leq} \frac{1}{\left|(\hat{S})^c\right|} \sum_{j \in(\hat{S})^c} v_j^2=\frac{1}{|I|-\hat{s}} \sum_{j \in(\hat{S})^c} v_j^2 \leq \frac{1}{|I|-\hat{s}} \sum_{j \in(\hat{S})^c}\|\hat{\boldsymbol{v}}-\boldsymbol{v}\|_2^2
$$
The (*) step is true because $T^c$ is the collection of indices with the smallest absolute values, and $\left|T^c\right| \leq\left|\hat{S}^c\right|$. We then combine the two displays above to conclude that

\begin{align*}
\left\|\tilde{P}_s(\boldsymbol{v})-\boldsymbol{v}\right\|_2^2 & \leq \sum_{j \in S^{\bullet} \cap T^c} v_j^2+(1+1 / c) \sum_{j \in S \cap T^c} v_j^2+4(1+c) \sum_{i \in[s]}\left\|\boldsymbol{w}_i\right\|_{\infty}^2 \\
& \leq(1+1 / c) \sum_{j \in T^c} v_j^2+4(1+c) \sum_{i \in[s]}\left\|\boldsymbol{w}_i\right\|_{\infty}^2\\
&\leq(1+1 / c) \frac{|I|-s}{|I|-\hat{s}}\|\hat{\boldsymbol{v}}-\boldsymbol{v}\|_2^2+4(1+c) \sum_{i \in[s]}\left\|\boldsymbol{w}_i\right\|_{\infty}^2
\end{align*}

\end{proof}

\begin{lemma}\label{IHTbd}
Under assumptions \ref{assumption1} and \ref{assumption2} and event $\mathcal{E}_1=\left\{\Pi_R\left(y_i\right)=y_i, \forall i \in[n]\right\}$, RSC and RSS properties implies that there exists an absolute constant $\rho$ such that
$$
\mathcal{L}_n\left(\boldsymbol{\beta}^{t+1}\right)-\mathcal{L}_n(\hat{\boldsymbol{\beta}}) \leq\left(1-\frac{c_l}{\rho c_u}\right)\left(\mathcal{L}_n\left(\boldsymbol{\beta}^t\right)-\mathcal{L}_n(\hat{\boldsymbol{\beta}})\right)+c_3\left(\sum_{i \in[s]}\left\|\boldsymbol{w}_i^t\right\|_{\infty}^2+\left\|\tilde{\boldsymbol{w}}_{S^{t+1}}^t\right\|_2^2\right)
$$
for every $t$, where $\boldsymbol{w}_1^t, \boldsymbol{w}_2^t, \cdots, \boldsymbol{w}_s^t$ are the Laplace noise added to $\boldsymbol{\beta}^t-\left(\eta^0 / n\right) \sum_{i=1}^n\left(\boldsymbol{x}_i^{\top} \boldsymbol{\beta}^t-\Pi_R\left(y_i\right)\right) \boldsymbol{x}_i$ when the support of $\boldsymbol{\beta}^{t+1}$ is iteratively selected by ``Peeling", $S^{t+1}$ is the support of $\beta^{t+1}$, and $\tilde{w}^t$ is the noise vector added to the selected s-sparse vector.
\end{lemma}

\begin{proof}
    For convenience, we notate as below:

$\triangleright$ Let $S^t=\operatorname{supp}\left(\boldsymbol{\beta}^t\right), S^{t+1}=\operatorname{supp}\left(\boldsymbol{\beta}^{t+1}\right), S^*=\operatorname{supp}(\hat{\boldsymbol{\beta}})$, and define $I^t=S^{t+1} \cup S^t \cup S^*$.

$\triangleright$ Let $g^t=\nabla \mathcal{L}_n\left(\beta^t\right)$ and $\eta^0=\eta / c_u$, where $c_u$ is the constant in RSC and RSS property.

$\triangleright$ Let $\boldsymbol{w}_1, \boldsymbol{w}_2, \cdots, \boldsymbol{w}_s$ be the noise vectors added to $\boldsymbol{\beta}^t-\eta^0 \nabla \mathcal{L}_n\left(\boldsymbol{\beta}^{\mathrm{t}}\right)$ when the support of $\boldsymbol{\beta}^{t+1}$ is iteratively selected. We define $\boldsymbol{W}=$ $4 \sum_{i \in[s]}\left\|\boldsymbol{w}_i\right\|_{\infty}^2$
By RSC and RSS property, we have
\begin{align}\label{eqq}
\mathcal{L}_n\left(\boldsymbol{\beta}^{t+1}\right)-\mathcal{L}_n\left(\boldsymbol{\beta}^t\right) & \leq\left\langle\boldsymbol{\beta}^{t+1}-\boldsymbol{\beta}^t, \boldsymbol{g}^t\right\rangle+\frac{c_u}{2}\left\|\boldsymbol{\beta}^{t+1}-\boldsymbol{\beta}^t\right\|_2^2 \\
 & =\frac{c_u}{2}\left\|\boldsymbol{\beta}_{I^t}^{t+1}-\boldsymbol{\beta}_{I^{+}}^t+\frac{\eta}{c_u} \boldsymbol{g}_{I^t}^t\right\|_2^2-\frac{\eta^2}{2 c_u}\left\|\boldsymbol{g}_{I^t}^t\right\|_2^2+(1-\eta)\left\langle\boldsymbol{\beta}^{t+1}-\boldsymbol{\beta}^t, \boldsymbol{g}^t\right\rangle .
\end{align}

We first focus on the third term above. In what follows, $c$ denotes an arbitrary constant greater than 1 . We may write $\boldsymbol{\beta}^{t+1}=\tilde{\boldsymbol{\beta}}^{t+1}+\tilde{\boldsymbol{w}}_{S^{t+1}}$, so that $\tilde{\boldsymbol{\beta}}^{t+1}=\tilde{P}_s\left(\boldsymbol{\beta}^t-\eta^0 \nabla \mathcal{L}_n\left(\boldsymbol{\beta}^t\right)\right)$ and $\tilde{\boldsymbol{w}}$ is a vector consisting of $d$ i.i.d. Laplace random variables.
\begin{align*}
\left\langle\boldsymbol{\beta}^{t+1}-\boldsymbol{\beta}^t, \boldsymbol{g}^t\right\rangle & =\left\langle\boldsymbol{\beta}_{S^{t+1}}^{t+1}-\boldsymbol{\beta}_{S^{t+1}}^t, \boldsymbol{g}_{S^{t+1}}^t\right\rangle-\left\langle\boldsymbol{\beta}_{S^t \backslash S^{t+1}}^t, \boldsymbol{g}_{S^t \backslash S^{t+1}}^t\right\rangle \\
& =\left\langle\overline{\boldsymbol{\beta}}_{S^{t+1}}^{t+1}-\boldsymbol{\beta}_{S^{t+1}}^t, \boldsymbol{g}_{S^{t+1}}^t\right\rangle+\left\langle\tilde{\boldsymbol{w}}_{S^{t+1}}, \boldsymbol{g}_{S^{t+1}}^t\right\rangle-\left\langle\boldsymbol{\beta}_{S^t \backslash S^{t+1}}^t, \boldsymbol{g}_{S^t \backslash S^{t+1}}^t\right\rangle
\end{align*}

It follows that, for every $c>1$, 
\begin{equation}\label{eq}
    \left\langle\boldsymbol{\beta}^{t+1}-\boldsymbol{\beta}^t, \boldsymbol{g}^t\right\rangle \leq-\frac{\eta}{c_u}\left\|\boldsymbol{g}_{S^{t+1}}^t\right\|_2^2+c\left\|\tilde{\boldsymbol{w}}_{S^{t+1}}\right\|_2^2+(1 / 4 c)\left\|\boldsymbol{g}_{S^{t+1}}^t\right\|_2^2-\left\langle\boldsymbol{\beta}_{S^t \backslash S^{t+1}}^t, \boldsymbol{g}_{S^t \backslash S^{t+1}}^t\right\rangle
\end{equation}

Now for the last term in the display above, we have
\begin{align*}
-\left\langle\boldsymbol{\beta}_{S^t \backslash S^{t+1}}^t, \boldsymbol{g}_{S^t \backslash S^{t+1}}^t\right\rangle & \leq \frac{c_u}{2 \eta}\left(\left\|\boldsymbol{\beta}_{S^t \backslash S^{t+1}}^t-\frac{\eta}{c_u} \boldsymbol{g}_{S^t \backslash S^{t+1}}^t\right\|_2^2-\left(\frac{\eta}{c_u}\right)^2\left\|\boldsymbol{g}_{S^t \backslash S^{t+1}}^t\right\|_2^2\right) \\
& \leq \frac{c_u}{2 \eta}\left\|\boldsymbol{\beta}_{S^t \backslash S^{t+1}}^t-\frac{\eta}{c_u} \boldsymbol{g}_{S^t \backslash S^{t+1}}^t\right\|_2^2-\frac{\eta}{2 c_u}\left\|\boldsymbol{g}_{S^t \backslash S^{t+1}}^t\right\|_2^2 .
\end{align*}
We apply lemma \ref{A.31} to $\left\|\boldsymbol{\beta}_{S^t \backslash S^{t+1}}-\frac{\eta}{c_u} \boldsymbol{g}_{S^t \backslash S^{t+1}}^t\right\|_2^2$ to obtain that, for every $c>1$
\begin{align*}
&-\left\langle\boldsymbol{\beta}_{S^t \backslash S^{t+1}}^t, \boldsymbol{g}_{S^t \backslash S^{t+1}}^t\right\rangle \leq \frac{c_u}{2 \eta}\left[(1+1 / c)\left\|\tilde{\boldsymbol{\beta}}_{S^{t+1} \backslash S^t}^{t+1}\right\|_2^2+(1+c) \boldsymbol{W}\right]-\frac{\eta}{2 c_u}\left\|\boldsymbol{g}_{S^t \backslash S^{t+1}}^t\right\|_2^2 \\
&=\frac{\eta}{2 c_u}\left[(1+1 / c)\left\|\boldsymbol{g}_{S^{t+1} \backslash S^t}^t\right\|_2^2+(1+c) \frac{c_u}{2 \eta} \boldsymbol{W}\right]-\frac{\eta}{2 c_u}\left\|\boldsymbol{g}_{S^t \backslash S^{t+1}}^t\right\|_2^2 .
\end{align*}
Plugging the inequality above back into (Eq.\ref{eq}) yields
\begin{align*}
\left\langle\boldsymbol{\beta}^{t+1}-\boldsymbol{\beta}^t, \boldsymbol{g}^t\right\rangle \leq & -\frac{\eta}{c_u}\left\|\boldsymbol{g}_{S^{t+1}}^t\right\|_2^2+c\left\|\tilde{\boldsymbol{w}}_{S^{t+1}}\right\|_2^2+(1 / 4 c)\left\|\boldsymbol{g}_{S^{t+1}}^t\right\|_2^2 \\
& +\frac{\eta}{2 c_u}\left[(1+1 / c)\left\|\boldsymbol{g}_{S^{t+1} \backslash S^t}^t\right\|_2^2+(1+c) \frac{c_u}{2 \eta} \boldsymbol{W}\right]-\frac{\eta}{2 c_u}\left\|\boldsymbol{g}_{S^t \backslash S^{t+1}}^t\right\|_2^2 \\
\leq & \frac{\eta}{2 c_u}\left\|\boldsymbol{g}_{S^{t+1} \backslash S^t}^t\right\|_2^2-\frac{\eta}{2 c_u}\left\|\boldsymbol{g}_{S^t \backslash S^{t+1}}^t\right\|_2^2-\frac{\eta}{c_u}\left\|\boldsymbol{g}_{S^{t+1}}^t\right\|_2^2 \\
& +(1 / c)\left(4+\frac{\eta}{2 c_u}\right)\left\|\boldsymbol{g}_{S^{t+1}}^t\right\|_2^2+c \overline{\boldsymbol{w}}_{S^{t+1}} \|_2^2+(1+c) \frac{c_u}{2 \eta} \boldsymbol{W}
\end{align*}
Finally, for the third term of (Eq.\ref{eqq}) we have
$$
\left\langle\beta^{t+1}-\boldsymbol{\beta}^t, \boldsymbol{g}^t\right\rangle \leq-\frac{\eta}{2 c_u}\left\|\boldsymbol{g}_{S^t \cup S^{t+1}}^t\right\|_2^2+(1 / c)\left(4+\frac{\eta}{2 c_u}\right)\left\|\boldsymbol{g}_{S^{t+1}}^t\right\|_2^2+c\left\|\tilde{\boldsymbol{w}}_{S^{t+1}}\right\|_2^2+(1+c) \frac{c_u}{2 \eta} \boldsymbol{W}
$$
Now combining this bound with (Eq.\ref{eqq}) yields
\begin{align*}
& \mathcal{L}_n\left(\boldsymbol{\beta}^{t+1}\right)-\mathcal{L}_n\left(\boldsymbol{\beta}^t\right) \\
\leq & \frac{c_u}{2}\left\|\boldsymbol{\beta}_{I^t}^{t+1}-\boldsymbol{\beta}_{I^t}^t+\frac{\eta}{c_u} \boldsymbol{g}_{I^t}^t\right\|_2^2-\frac{\eta^2}{2 c_u}\left\|\boldsymbol{g}_{I^t}^t\right\|_2^2-\frac{\eta(1-\eta)}{2 c_u}\left\|\boldsymbol{g}_{S^t \cup S^{t+1}}^t\right\|_2^2 \\
& +\frac{1-\eta}{c}\left(4+\frac{\eta}{2 c_u}\right)\left\|\boldsymbol{g}_{S^{t+1}}^t\right\|_2^2+(1-\eta) c\left\|\overline{\boldsymbol{w}}_{S^{t+1}}\right\|_2^2+(1-\eta)(1+c) \frac{c_u}{2 \eta} \boldsymbol{W} \\
\leq & \frac{c_u}{2}\left\|\boldsymbol{\beta}_{I^t}^{t+1}-\boldsymbol{\beta}_{I^t}^t+\frac{\eta}{c_u} \boldsymbol{g}_{I^t}^t\right\|_2^2-\frac{\eta^2}{2 c_u}\left\|\boldsymbol{g}_{I^t \backslash\left(S^t \cup S^*\right)}^t\right\|_2^2-\frac{\eta^2}{2 c_u}\left\|\boldsymbol{g}_{S^t \cup S^*}^t\right\|_2^2-\frac{\eta(1-\eta)}{2 c_u}\left\|\boldsymbol{g}_{S^t \cup S^{t+1}}^t\right\|_2^2 \\
& +\frac{1-\eta}{c}\left(4+\frac{\eta}{2 c_u}\right)\left\|\boldsymbol{g}_{S^{t+1}}^t\right\|_2^2+(1-\eta) c\left\|\tilde{\boldsymbol{w}}_{S^{t+1}}\right\|_2^2+(1-\eta)(1+c) \frac{c_u}{2 \eta} \boldsymbol{W} \\
\leq & \frac{c_u}{2}\left\|\boldsymbol{\beta}_{I^t}^{t+1}-\boldsymbol{\beta}_{I^t}^t+\frac{\eta}{c_u} \boldsymbol{g}_{I^t}^t\right\|_2^2-\frac{\eta^2}{2 c_u}\left\|\boldsymbol{g}_{I^t \backslash\left(S^t \cup S^*\right)}^t\right\|_2^2-\frac{\eta^2}{2 c_u}\left\|\boldsymbol{g}_{S^t \cup S^*}^t\right\|_2^2-\frac{\eta(1-\eta)}{2 c_u}\left\|\boldsymbol{g}_{S^{t+1} \backslash\left(S^t \cup S^*\right)}^t\right\|_2^2\\
&+\frac{1-\eta}{c}\left(4+\frac{\eta}{2 c_u}\right)\left\|\boldsymbol{g}_{S^{t+1}}^t\right\|_2^2+(1-\eta) c\left\|\overline{\boldsymbol{w}}_{S^{t+1}}\right\|_2^2+(1-\eta)(1+c) \frac{c_u}{2 \eta} \boldsymbol{W}.
\end{align*}

The last step is true because $S^{t+1} \backslash\left(S^t \cup S^*\right)$ is a subset of $S^t \cup S^{t+1}$. We next analyze the first two terms, $\frac{c_u}{2}\left\|\boldsymbol{\beta}_{I^t}^{t+1}-\boldsymbol{\beta}_{I^t}^t+\frac{\eta}{c_u} \boldsymbol{g}_{I^t}^t\right\|_2^2-\frac{\eta^2}{2 c_u}\left\|\boldsymbol{g}_{I^t \backslash\left(S^t \cup S^*\right)}^t\right\|_2^2$.
Let $R$ be a subset of $S^t \backslash S^{t+1}$ such that $|R|=\left|I^t \backslash\left(S^t \cup S^*\right)\right|=\mid S^{t+1} \backslash\left(S^t \cup\right.$ $\left.S^*\right) \mid$. 

By the definition of $\tilde{\boldsymbol{\beta}}^{t+1}$ and lemma \ref{A.31}, we have, for every $c>1$,
$$
\frac{\eta^2}{c_u^2}\left\|\boldsymbol{g}_{I^t \backslash\left(S^t \cup S^*\right)}^t\right\|_2^2=\left\|\tilde{\boldsymbol{\beta}}_{I^t \backslash\left(S^t \cup S^*\right)}^{t+1}\right\|_2^2 \geq(1-1 / c)\left\|\boldsymbol{\beta}_R^t-\frac{\eta}{c_u} \boldsymbol{g}_R^t\right\|_2^2-c \boldsymbol{W} .
$$
It follows that
\begin{align*}
& \frac{c_u}{2}\left\|\boldsymbol{\beta}_{I^t}^{t+1}-\boldsymbol{\beta}_{I^t}^t+\frac{\eta}{c_u} \boldsymbol{g}_{I^t}^t\right\|_2^2-\frac{\eta^2}{2 c_u}\left\|\boldsymbol{g}_{I^t \backslash\left(S^t \cup S^*\right)}^t\right\|_2^2 \\
& \leq \frac{c_u}{2}\left\|\tilde{\boldsymbol{w}}_{S^{t+1}}\right\|_2^2+\frac{c_u}{2}\left\|\tilde{\boldsymbol{\beta}}_{I^t}^{t+1}-\boldsymbol{\beta}_{I^t}^t+\frac{\eta}{c_u} \boldsymbol{g}_{I^t}^t\right\|_2^2-\frac{c_u}{2}(1-1 / c)\left\|\boldsymbol{\beta}_R^t-\frac{\eta}{c_u} \boldsymbol{g}_R^t\right\|_2^2+\frac{c c_u}{2} \boldsymbol{W} \\
& =\frac{c_u}{2}\left\|\tilde{\boldsymbol{\beta}}_{I^t}^{t+1}-\boldsymbol{\beta}_{I^t}^t+\frac{\eta}{c_u} \boldsymbol{g}_{I^t}^t\right\|_2^2-\frac{c_u}{2}\left\|\mid \tilde{\boldsymbol{\beta}}_R^{t+1}-\boldsymbol{\beta}_R^t+\frac{\eta}{c_u} \boldsymbol{g}_R^t\right\|_2^2+\frac{c_u}{2}(1 / c)\left\|\boldsymbol{\beta}_R^t-\frac{\eta}{c_u} \boldsymbol{g}_R^t\right\|_2^2+\frac{c c_u}{2} \boldsymbol{W} \\
& +\frac{c_u}{2}\left\|\tilde{\boldsymbol{w}}_{S^{t+1}}\right\|_2^2 \\
& \leq \frac{c_u}{2}\left\|\tilde{\boldsymbol{\beta}}_{I^t \backslash R}^{t+1}-\boldsymbol{\beta}_{I^t \backslash R}^t+\frac{\eta}{c_u} \boldsymbol{g}_{I^{\dagger} \backslash R}^t\right\|_2^2+\frac{\eta^2}{2 c c_u}(1+1 / c)\left\|\boldsymbol{g}_{I^t \backslash\left(S^t \cup S^*\right)}^t\right\|_2^2+\frac{c c_u}{2} \boldsymbol{W}+\frac{c_u}{2}\left\|\tilde{\boldsymbol{w}}_{S^{t+1}}\right\|_2^2 .
\end{align*}

The last inequality is obtained by applying lemma \ref{A.31} to $\left\|\boldsymbol{\beta}_R^t-\frac{\eta}{c_u} \boldsymbol{g}_R^t\right\|_2^2$. Now we apply lemma \ref{A.3} to obtain
\begin{align*}
& \frac{c_u}{2}\left\|\boldsymbol{\beta}_{I^t}^{t+1}-\boldsymbol{\beta}_{I^t}^t+\frac{\eta}{c_u} \boldsymbol{g}_{I^t}^t\right\|_2^2-\frac{\eta^2}{2 c_u}\left\|\boldsymbol{g}_{I^t \backslash\left(S^t \cup S^*\right)}^t\right\|_2^2 \\
& \leq \frac{3 c_u}{4} \frac{\left|I^t \backslash R\right|-s}{\left|I^t \backslash R\right|-s^*}\left\|\boldsymbol{\beta}_{I^t \backslash R}-\boldsymbol{\beta}_{I^t \backslash R}^t+\frac{\eta}{c_u} \boldsymbol{g}_{I^t \backslash R}^t\right\|_2^2+\frac{3 c_u}{2} \boldsymbol{W} \\
& +\frac{\eta^2\left(1+c^{-1}\right)}{2 c c_u}\left\|\boldsymbol{g}_{I^{\backslash} \backslash\left(S^t \cup S^*\right)}^t\right\|_2^2+\frac{c c_u}{2} \boldsymbol{W}+\frac{c_u}{2}\left\|\tilde{\boldsymbol{w}}_{S^{t+1}}\right\|_2^2 \\
& \leq \frac{3 c_u}{4} \frac{2 s^*}{s+s^*}\left\|\boldsymbol{\beta}_{I^t \backslash R}-\boldsymbol{\beta}_{I^t \backslash R}^t+\frac{\eta}{c_u} \boldsymbol{g}_{I^t \backslash R}^t\right\|_2^2+\frac{3 c_u}{2} \boldsymbol{W}+\frac{\eta^2}{2 c c_u}(1+1 / c)\left\|\boldsymbol{g}_{S^{t+1}}^t\right\|_2^2 \\
&+\frac{c c_u}{2} \boldsymbol{W}+\frac{c_u}{2}\left\|\tilde{\boldsymbol{w}}_{S^{t+1}}\right\|_2^2
\end{align*}

The last step is true by observing that $\left|I^t \backslash R\right| \leq 2 s^*+s$, and the inclusion $I^t \backslash\left(S^t \cup S^*\right) \subseteq S^{t+1}$. We continue to simplify,
\begin{align*}
&\frac{c_u}{2}\left\|\boldsymbol{\beta}_{I^t}^{t+1}-\boldsymbol{\beta}_{I^t}^t+\frac{\eta}{c_u} \boldsymbol{g}_{I^t}^t\right\|_2^2-\frac{\eta^2}{2 c_u}\left\|\boldsymbol{g}_{I^t \backslash\left(S^t \cup S^*\right)}^t\right\|_2^2\\
& \leq \frac{c_u}{2} \frac{3 s^*}{s+s^*}\left\|\boldsymbol{\beta}_{I^t}-\boldsymbol{\beta}_{I^t}^t+\frac{\eta}{c_u} \boldsymbol{g}_{I^t}^t\right\|_2^2+\frac{3 c_u}{2} \boldsymbol{W}+\frac{\eta^2}{2 c c_u}(1+1 / c)\left\|\boldsymbol{g}_{S^{t+1}}^t\right\|_2^2+\frac{c c_u}{2} \boldsymbol{W}+\frac{c_u}{2}\left\|\tilde{\boldsymbol{w}}_{S^{t+1}}\right\|_2^2 \\
& \leq \frac{3 s^*}{s+s^*}\left(\eta\left(\hat{\boldsymbol{\beta}}-\boldsymbol{\beta}^t, \boldsymbol{g}^t\right\rangle+\frac{c_u}{2}\left\|\hat{\boldsymbol{\beta}}-\boldsymbol{\beta}^t\right\|_2^2+\frac{\eta^2}{2 c c_u}\left\|\boldsymbol{g}_{I^t}^t\right\|_2^2\right) \\
& +\frac{\eta^2}{2 c c_u}(1+1 / c)\left\|\boldsymbol{g}_{S^{t+1}}^t\right\|_2^2+\frac{(c+3) c_u}{2} \boldsymbol{W}+\frac{c_u}{2}\left\|\overline{\boldsymbol{w}}_{S^{t+1}}\right\|_2^2 \\
& \leq \frac{3 s^*}{s+s^*}\left(\eta \mathcal{L}_n(\hat{\boldsymbol{\beta}})-\eta \mathcal{L}_n\left(\boldsymbol{\beta}^t\right)+\frac{c_u-\eta c_l}{2}\left\|\hat{\boldsymbol{\beta}}-\boldsymbol{\beta}^t\right\|_2^2+\frac{\eta^2}{2 c c_u}\left\|\boldsymbol{g}_{I^t}^t\right\|_2^2\right) \\
& \quad+\frac{\eta^2}{2 c c_u}(1+1 / c)\left\|\boldsymbol{g}_{S^{t+1}}^t\right\|_2^2+\frac{(c+3) c_u}{2} \boldsymbol{W}+\frac{c_u}{2}\left\|\tilde{\boldsymbol{w}}_{S^{t+1}}\right\|_2^2 .
\end{align*}

Until now, the inequality is true for any $0<\eta<1$ and $c>1$. We now specify the choice of these parameters: let $\eta=2 / 3$ and set $c$ large enough so that
\begin{align*}
\mathcal{L}_n\left(\boldsymbol{\beta}^{t+1}\right)-\mathcal{L}_n\left(\boldsymbol{\beta}^t\right) \leq & \frac{3 s^*}{s+s^*}\left(\eta \mathcal{L}_n(\hat{\boldsymbol{\beta}})-\eta \mathcal{L}_n\left(\boldsymbol{\beta}^t\right)+\frac{c_u-\eta c_l}{2}\left\|\hat{\boldsymbol{\beta}}-\boldsymbol{\beta}^t\right\|_2^2+\frac{\eta^2}{2 c_u}\left\|\boldsymbol{g}_{I^t}^t\right\|_2^2\right) \\
& -\frac{\eta^2}{4 c_u}\left\|\boldsymbol{g}_{S^t \cup S^*}^t\right\|_2^2-\frac{\eta(1-\eta)}{4 c_u}\left\|\boldsymbol{g}_{S^{t+1} \backslash\left(S^t \cup S^*\right)}^t\right\|_2^2 \\
& +\frac{c_u}{2}\left(\frac{3 c+7}{2}\right) \boldsymbol{W}+\left(\frac{c}{3}+\frac{c_u}{2}\right)\left\|\tilde{\boldsymbol{w}}_{S^{t+1}}\right\|_2^2 .
\end{align*}

Such a choice of $c$ is available because $c_u$ is bounded above by an absolute constant thanks to the RSM condition (upper inequality of RSC and RSS property). Now we set $s=72(c_u / c_l)^2 s^*=\rho L^4 s^*$, where $\rho$ is the absolute constant referred to in Lemma 8.3 and Theorem 4.4 in \cite{sun2020adaptive}, so that $\frac{3 s^*}{s+s^*} \leq \frac{c_l^2}{24 c_u(c_u-\eta c_l)}$, and $\frac{c_l^2}{24 c_u(c_u-\eta c_l)} \leq 1 / 8$ because $c_l<c_u$. It follows that
\begin{align*}
\mathcal{L}_n\left(\boldsymbol{\beta}^{t+1}\right)-\mathcal{L}_n\left(\boldsymbol{\beta}^t\right) \leq & \frac{3 s^*}{s+s^*}\left(\eta \mathcal{L}_n(\hat{\boldsymbol{\beta}})-\eta \mathcal{L}_n\left(\boldsymbol{\beta}^t\right)\right)+\frac{c_l^2}{48 c_u}\left\|\hat{\boldsymbol{\beta}}-\boldsymbol{\beta}^t\right\|_2^2+\frac{1}{36 c_u}\left\|\boldsymbol{g}_{I^t}^t\right\|_2^2 \\
& -\frac{1}{9 c_u}\left\|\boldsymbol{g}_{S^t \cup S^*}^t\right\|_2^2-\frac{1}{18 c_u}\left\|\boldsymbol{g}_{S^{t+1} \backslash\left(S^t \cup S^*\right)}^t\right\|_2^2 \\
& +\frac{c_u}{2}\left(\frac{3 c+7}{2}\right) \boldsymbol{W}+\left(\frac{c}{3}+\frac{c_u}{2}\right)\left\|\tilde{\boldsymbol{w}}_{S^{t+1}}\right\|_2^2
\end{align*}

Because $\left\|\boldsymbol{g}_{I^t}^t\right\|_2^2=\left\|\boldsymbol{g}_{S^t \cup S^*}^t\right\|_2^2+\left\|\boldsymbol{g}_{S^{t+1} \backslash\left(S^t \cup S^*\right)}^t\right\|_2^2$, we have
\begin{align*}
\mathcal{L}_n\left(\boldsymbol{\beta}^{t+1}\right)-\mathcal{L}_n\left(\boldsymbol{\beta}^t\right) \leq & \frac{3 s^*}{s+s^*}\left(\eta \mathcal{L}_n(\hat{\boldsymbol{\beta}})-\eta \mathcal{L}_n\left(\boldsymbol{\beta}^t\right)\right)+\frac{c_l^2}{48 c_u}\left\|\hat{\boldsymbol{\beta}}-\boldsymbol{\beta}^t\right\|_2^2-\frac{3}{36 c_u}\left\|\boldsymbol{g}_{S^t \cup S^*}^t\right\|_2^2 \\
& +\frac{c_u}{2}\left(\frac{3 c+7}{2}\right) \boldsymbol{W}+\left(\frac{c}{3}+\frac{c_u}{2}\right)\left\|\overline{\boldsymbol{w}}_{S^{t+1}}\right\|_2^2 \\
\leq & \frac{3 s^*}{s+s^*}\left(\eta \mathcal{L}_n(\hat{\boldsymbol{\beta}})-\eta \mathcal{L}_n\left(\boldsymbol{\beta}^t\right)\right)-\frac{3}{36 c_u}\left(\left\|\boldsymbol{g}_{S^t \cup S^*}^t\right\|_2^2-\frac{c_l^2}{4}\left\|\hat{\boldsymbol{\beta}}-\boldsymbol{\beta}^t\right\|_2^2\right)\\
&+\frac{c_u}{2}\left(\frac{3 c+7}{2}\right) \boldsymbol{W}+\left(\frac{c}{3}+\frac{c_u}{2}\right)\left\|\overline{\boldsymbol{w}}_{S^{t+1}}\right\|_2^2 \ \text{(Eq.3)}.
\end{align*}

To continue the calculations, we consider a lemma from Lemma A.4\cite{c26}
$$
\left\|g_{S^t \cup S^*}^t\right\|_2^2-\frac{c_l^2}{4}\left\|\hat{\boldsymbol{\beta}}-\boldsymbol{\beta}^t\right\|_2^2 \geq \frac{c_l}{2}\left(\mathcal{L}_n\left(\boldsymbol{\beta}^t\right)-\mathcal{L}_n(\hat{\boldsymbol{\beta}})\right)
$$
It then follows from (Eq.3), the quoted lemma above and the definition of $\rho$ that, for an appropriate constant $c_3$,
\begin{align*}
\mathcal{L}_n\left(\boldsymbol{\beta}^{t+1}\right)-\mathcal{L}_n\left(\boldsymbol{\beta}^t\right) & \leq-\left(\frac{3 c_l}{72 c_u}+\frac{2 s^*}{s+s^*}\right)\left(\mathcal{L}_n\left(\boldsymbol{\beta}^t\right)-\mathcal{L}_n(\hat{\boldsymbol{\beta}})\right)+c_3\left(\boldsymbol{W}+\left\|\overline{\boldsymbol{w}}_{S^{t+1}}\right\|_2^2\right) \\
& \leq-\left(\frac{c_l}{\rho c_u}\right)\left(\mathcal{L}_n\left(\boldsymbol{\beta}^t\right)-\mathcal{L}_n(\hat{\boldsymbol{\beta}})\right)+c_3\left(\boldsymbol{W}+\left\|\tilde{\boldsymbol{w}}_{S^{t+1}}\right\|_2^2\right)
\end{align*}

Adding $\mathcal{L}_n\left(\boldsymbol{\beta}^t\right)-\mathcal{L}_n(\hat{\boldsymbol{\beta}})$ to both sides of the inequality concludes the proof.
\end{proof}

\section{Proof of Theorem \ref{theorem2}}
\begin{proof}
Using the lemmas stated above, we iterate over $t$ and denote
\[
\boldsymbol{W}_t = c_3\left(\sum_{i\in[s]}\|\boldsymbol{w}_i^t\|_{\infty}^2 + \|\tilde{\boldsymbol{w}}_{S^{t+1}}^t\|_2^2\right).
\]
Then, we have
\begin{align*}
\mathcal{L}_n(\boldsymbol{\beta}^T)-\mathcal{L}_n(\boldsymbol{\beta}^*)
&\le \left(1-\frac{c_l}{\rho c_u}\right)^T \Bigl(\mathcal{L}_n(\boldsymbol{\beta}^0)-\mathcal{L}_n(\boldsymbol{\beta}^*)\Bigr)
+ \sum_{k=0}^{T-1}\left(1-\frac{c_l}{\rho c_u}\right)^{T-k-1} \boldsymbol{W}_k \\
&\le \left(1-\frac{c_l}{\rho c_u}\right)^T 8Ac_0^2
+ \sum_{k=0}^{T-1}\left(1-\frac{c_l}{\rho c_u}\right)^{T-k-1} \boldsymbol{W}_k,
\end{align*}
where the second inequality follows from the RSS and RSC properties and the $\ell_2$ bounds on $\boldsymbol{\beta}^0$ and $\boldsymbol{\beta}^*$.

On the other hand, by the RSC property we also have the lower bound
\[
\mathcal{L}_n(\boldsymbol{\beta}^T)-\mathcal{L}_n(\boldsymbol{\beta}^*)
\ge \frac{c_l}{2}\|\boldsymbol{\beta}^T-\boldsymbol{\beta}^*\|_2^2
-\left\langle \nabla \mathcal{L}_n(\boldsymbol{\beta}^*),\, \boldsymbol{\beta}^*-\boldsymbol{\beta}^T \right\rangle.
\]
Combining these bounds and noting that for $T\asymp \log n$, we obtain
\begin{equation}\label{eq1}
\frac{c_l}{2}\|\boldsymbol{\beta}^T-\boldsymbol{\beta}^*\|_2^2
\le \|\nabla \mathcal{L}_n(\boldsymbol{\beta}^*)\|_{\infty}\sqrt{s+s^*}\|\boldsymbol{\beta}^*-\boldsymbol{\beta}^T\|_2
+\frac{1}{n} + \sum_{k=0}^{T-1}\left(1-\frac{c_l}{\rho c_u}\right)^{T-k-1} \boldsymbol{W}_k.
\end{equation}

To further bound $\|\boldsymbol{\beta}^T-\boldsymbol{\beta}^*\|_2^2$, we consider the event
\[
\mathcal{E}_2 = \left\{ \max_t \boldsymbol{W}_t \le K\,\frac{L^2\tau^2(s^*)^2 \log^2 d\,\log\left(\frac{1}{\sigma}\right)t^2}{n^2\varepsilon^2} \right\},
\]
in addition to the event $\mathcal{E}_1$ defined earlier. Then, by applying Eq.~\eqref{eq1}, Assumptions \ref{assumption1} and \ref{assumption2}, and Lemma \ref{IHTbd}, we deduce that
\[
\|\boldsymbol{\beta}^T-\boldsymbol{\beta}^*\|_2 \lesssim (s^*)^{\frac{1}{2}}\left\|\nabla \mathcal{L}_\tau(\boldsymbol{\beta}^*)\right\|_{\infty}
+\frac{L\tau s^*\,\log d\,\log^{\frac{1}{2}}\left(\frac{1}{\sigma}\right)t}{n\varepsilon}.
\]

Regardless of whether the noise is random or fixed, the leading term of $\nabla \mathcal{L}_\tau(\boldsymbol{\beta}^*)$ is given by
\[
C_1\tau L\frac{t}{n} + C_2\tau^{-\delta},
\]
as established in Lemma \ref{newl}. Therefore, the optimal tuning for $\tau$ is
\[
\tau \asymp \left(\frac{t}{n}\left(1+\frac{s^{*\frac{1}{2}}\log d\,\log^{\frac{1}{2}}\left(\frac{1}{\sigma}\right)}{\varepsilon}\right)\right)^{-\max\left\{\frac{1}{2},\frac{1}{1+\zeta}\right\}}.
\]

It remains to show that the events $\mathcal{E}_1$ and $\mathcal{E}_2$ occur with high probability. As shown in \cite{cai2021cost}, we have
\[
\mathbb{P}\left(\mathcal{E}_1^c\right) \le c_1 \exp\left(-c_2\,t\log d\right)
\quad \text{and} \quad
\mathbb{P}\left(\mathcal{E}_3^c\right) \le 2e^{-2\log d},
\]
where $\mathcal{E}_3$ is another event defined in the previous context. This completes the proof.
\end{proof}

\section{Proof of Theorem \ref{thm:two-phase}}

\begin{proof}
Let $\boldsymbol{\beta}^0$ denote the initial estimate and $\boldsymbol{\beta}^*$ the true parameter vector. Define
\[
D_t \;=\; (1 - c_1)^t\,\|\boldsymbol{\beta}^0 - \boldsymbol{\beta}^*\|_2,
\]
where $c_1\in (0,1)$ is the contraction constant determined by the algorithm's specifications. We prove by induction on $t$ that
\[
\|\boldsymbol{\beta}^t - \boldsymbol{\beta}^*\|_2 \;\le\; D_t,
\]
up to an additive error of order $O\bigl(\sqrt{\boldsymbol{W}}\bigr)$.

\medskip
\noindent\textbf{Base Case.} For $t=0$, the claim is immediate:
\[
\|\boldsymbol{\beta}^0 - \boldsymbol{\beta}^*\|_2 = D_0.
\]

\medskip
\noindent\textbf{Inductive Step.} Suppose that 
\[
\|\boldsymbol{\beta}^t - \boldsymbol{\beta}^*\|_2 \le D_t
\]
for some $t \ge 0$. We now show that
\[
\|\boldsymbol{\beta}^{t+1} - \boldsymbol{\beta}^*\|_2 \;\le\; D_{t+1} + O\bigl(\sqrt{\boldsymbol{W}}\bigr).
\]
According to the algorithm, the update is performed in two stages. First, a gradient descent step is executed:
\[
\boldsymbol{\beta}^{t+0.5} \;=\; \boldsymbol{\beta}^t - \eta_t\,\mathbf{G}_t,
\]
where $\mathbf{G}_t$ is a subgradient of $f(\boldsymbol{\beta})$ evaluated at $\boldsymbol{\beta}^t$. By the parallelogram identity, we have
\[
\|\boldsymbol{\beta}^{t+0.5} - \boldsymbol{\beta}^*\|_2^2 
=\|\boldsymbol{\beta}^t - \boldsymbol{\beta}^*\|_2^2 - 2\eta_t\,\langle \boldsymbol{\beta}^t - \boldsymbol{\beta}^*, \mathbf{G}_t \rangle + \eta_t^2\,\|\mathbf{G}_t\|_2^2.
\]
Since $f$ is convex, it follows that
\[
\langle \boldsymbol{\beta}^t - \boldsymbol{\beta}^*, \mathbf{G}_t \rangle \ge f(\boldsymbol{\beta}^t)-f(\boldsymbol{\beta}^*).
\]
Thus, 
\[
\|\boldsymbol{\beta}^{t+0.5} - \boldsymbol{\beta}^*\|_2^2 \le \|\boldsymbol{\beta}^t - \boldsymbol{\beta}^*\|_2^2 - 2\eta_t\Bigl[f(\boldsymbol{\beta}^t)-f(\boldsymbol{\beta}^*)\Bigr] + \eta_t^2\,\|\mathbf{G}_t\|_2^2.
\]

By Lemma~1 in \cite{shen2023computationally}, there exists a constant $c_l>0$ such that
\[
f(\boldsymbol{\beta}^t)-f(\boldsymbol{\beta}^*) \ge \frac{n}{4}\,c_l^{1/2}\,\|\boldsymbol{\beta}^t-\boldsymbol{\beta}^*\|_2,
\]
and, moreover, $\|\mathbf{G}_t\|_2 \le n\,c_u^{1/2}$ for some constant $c_u>0$. Substituting these bounds gives
\[
\|\boldsymbol{\beta}^{t+0.5} - \boldsymbol{\beta}^*\|_2^2 \le \|\boldsymbol{\beta}^t - \boldsymbol{\beta}^*\|_2^2 - \eta_t\,\frac{n}{2}\,c_l^{1/2}\,\|\boldsymbol{\beta}^t-\boldsymbol{\beta}^*\|_2 + \eta_t^2\,n^2\,c_u.
\]
By the induction hypothesis, $\|\boldsymbol{\beta}^t-\boldsymbol{\beta}^*\|_2 \le D_t$, so that
\[
\|\boldsymbol{\beta}^{t+0.5} - \boldsymbol{\beta}^*\|_2^2 \le D_t^2 - \eta_t\,\frac{n}{2}\,c_l^{1/2}\,D_t + \eta_t^2\,n^2\,c_u.
\]

To obtain $\boldsymbol{\beta}^{t+1}$, the algorithm applies a projection step that incorporates privacy-preserving noise. Let $\boldsymbol{W}_t$ denote the noise introduced in iteration $t$, so that by Lemma~\ref{A.3} (or a similar projection result),
\[
\|\boldsymbol{\beta}^{t+1} - \boldsymbol{\beta}^*\|_2^2 \le \|\boldsymbol{\beta}^{t+0.5} - \boldsymbol{\beta}^*\|_2^2 + \boldsymbol{W}_t,
\]
where $\boldsymbol{W}_t = 4\sum_{i=1}^{s}\|w_i^t\|_\infty^2$. Consequently,
\[
\|\boldsymbol{\beta}^{t+1} - \boldsymbol{\beta}^*\|_2^2 \le D_t^2 - \eta_t\,\frac{n}{2}\,c_l^{1/2}\,D_t + \eta_t^2\,n^2\,c_u + \boldsymbol{W}_t.
\]

By choosing the step size such that
\[
\eta_t \in n^{-1}\sqrt{\frac{c_l}{c_u}}\,\Bigl[\tfrac{1}{8}D_t,\;\tfrac{3}{8}D_t\Bigr],
\]
the negative (linear) term dominates the quadratic term, ensuring a geometric decrease of $D_t$, provided that $\boldsymbol{W}_t$ (and hence the overall noise) remains sufficiently small with high probability. Taking square roots and using standard inequalities, we deduce
\[
\|\boldsymbol{\beta}^{t+1} - \boldsymbol{\beta}^*\|_2 \le (1-C_1)\,D_t + \sqrt{\boldsymbol{W}_t} \le D_{t+1} + \sqrt{\boldsymbol{W}},
\]
where
\[
O(\boldsymbol{W}) = O\!\Biggl(\frac{(s^*)^{3/2}\,\log d\,(\log(1/\delta))^{1/2}\,\log n}{n\,\varepsilon}\Biggr)
\]
serves as an upper bound on the added noise. This completes the induction for Phase One.

\medskip
\noindent\textbf{Transition to Phase Two.} Once $\|\boldsymbol{\beta}^t - \boldsymbol{\beta}^*\|_2$ becomes sufficiently small, the algorithm enters Phase Two, in which a constant step size is adopted. The analysis in this phase follows analogously to that of Phase One (again leveraging Lemma~1 in \cite{shen2023computationally}), and yields the contraction
\[
\|\boldsymbol{\beta}_{l+1} - \boldsymbol{\beta}^*\|_2 \le \Bigl(1 - \widetilde{c}_2\,\frac{c_l^2}{2\,c_u^2}\,\frac{b_1^2}{b_0^2}\Bigr)\,\|\boldsymbol{\beta}_l - \boldsymbol{\beta}^*\|_2,
\]
where the constants are chosen appropriately.

\medskip
Finally, by ensuring that additional perturbations (e.g., from noise or minor discrepancies in the index sets) remain controlled (as in Theorem~9 of \cite{shen2023computationally}), we conclude that the overall algorithm converges geometrically. This completes the proof of Theorem~\ref{thm:two-phase}.
\end{proof}

\section{Proof of Theorem \ref{thm:error-bound}}

\begin{proof}
The proof follows by combining the two-phase convergence analysis from Theorem~\ref{thm:two-phase}.

\textbf{Phase One:} When 
\[
\|\boldsymbol{\beta}^t - \boldsymbol{\beta}^*\|_2 \ge 8\,c_l^{-1/2}\gamma,
\]
with the decaying step size $\eta_t = (1-c_1)^t\eta_0$, Theorem~\ref{thm:two-phase} implies that
\[
\|\boldsymbol{\beta}^{t+1} - \boldsymbol{\beta}^*\|_2 \le (1-c_1)^{t+1}\|\boldsymbol{\beta}_0 - \boldsymbol{\beta}^*\|_2 + O\bigl(\sqrt{\boldsymbol{W}}\bigr).
\]
Thus, after
\[
T_1 = O\Bigl(\log\Bigl(\frac{\|\boldsymbol{\beta}_0 - \boldsymbol{\beta}^*\|_2}{\gamma}\Bigr)\Bigr)
\]
iterations, the error is reduced to
\[
\|\boldsymbol{\beta}^{T_1} - \boldsymbol{\beta}^*\|_2 \le 8\,c_l^{-1/2}\gamma.
\]

\textbf{Phase Two:} Once the iterate enters the region 
\[
\|\boldsymbol{\beta}^t - \boldsymbol{\beta}^*\|_2 \le 8\,c_l^{-1/2}\gamma,
\]
a constant step size is employed so that
\[
\|\boldsymbol{\beta}^{t+1} - \boldsymbol{\beta}^*\|_2 \le (1-c_2^*)\,\|\boldsymbol{\beta}^t - \boldsymbol{\beta}^*\|_2 +  O\bigl(\sqrt{\boldsymbol{W}}\bigr)\bigr).
\]
This contraction implies that after an additional
\[
T_2 = O\Bigl(\log\Bigl(n\,\gamma\,b_0^{-1}\,\log\Bigl(\frac{2\,d}{s^*}\Bigr)\Bigr)\Bigr)
\]
iterations the estimation error satisfies
\[
\|\boldsymbol{\beta}^{T_1+T_2} - \boldsymbol{\beta}^*\|_2 \le  O\bigl(\sqrt{\boldsymbol{W}}\bigr)\bigr).
\]

\textbf{Conclusion:} Combining the two phases, the total number of iterations required is
\[
T = T_1 + T_2 = O\Bigl(\log\Bigl(\frac{\|\boldsymbol{\beta}_0 - \boldsymbol{\beta}^*\|_2}{\gamma}\Bigr) + \log\Bigl(n\,\gamma\,b_0^{-1}\,\log\Bigl(\frac{2\,d}{s^*}\Bigr)\Bigr)\Bigr),
\]
and the final estimator $\boldsymbol{\beta}^T$ satisfies
\[
\|\boldsymbol{\beta}^T - \boldsymbol{\beta}^*\|_2 \le O\Bigl(\frac{(s^*)^{3/2}\,\log d\,\bigl(\log(1/\delta)\bigr)^{1/2}\,\log (T/n)}{(T/n)\,\varepsilon}\Bigr).
\]
Notice that the term 
\[
O\Bigl(\frac{(s^*)^{3/2}\,\log d\,\bigl(\log(1/\delta)\bigr)^{1/2}\,\log (T/n)}{(T/n)\,\varepsilon}\Bigr)
\]
originates from the cumulative effect of the Laplace noise injected via the peeling procedure. More precisely, under the event
\[
\mathcal{E}_2 = \left\{\max_t \boldsymbol{W}_t \leq K\,\frac{L^2 \tau^2\left(s^*\right)^2 \log^2 d\,\log\left(\frac{1}{\sigma}\right)\,t^2}{(T/n)^2\varepsilon^2}\right\},
\]
which holds with high probability, the noise magnitude in each iteration is well controlled.
This completes the proof.
\end{proof}

\end{document}